\documentclass{article}

\usepackage{longcat_style}
\usepackage[utf8]{inputenc}
\usepackage[T1]{fontenc}
\usepackage{microtype}
\usepackage{amsthm}
\newtheorem{proposition}{Proposition}
\usepackage{booktabs}
\usepackage{placeins}
\usepackage{flafter}
\usepackage{xcolor}
\usepackage{natbib}
\usepackage[hidelinks]{hyperref}
\hypersetup{
  pdftitle={Long-Horizon Scaling: How Model Capabilities Shape the Returns to Computation},
  pdfauthor={Haoyu Zheng, Zhengyu Chen, Huaisheng Zhu, Ruishan Fang, Teng Xiao, Yiwei Li, Jingang Wang, Wenqiao Zhang}
}
\usepackage{url}

\usepackage{amsmath,amsfonts,bm}

\def\eqref#1{equation~\ref{#1}}

\def\1{\bm{1}}

\DeclareMathAlphabet{\mathsfit}{\encodingdefault}{\sfdefault}{m}{sl}
\SetMathAlphabet{\mathsfit}{bold}{\encodingdefault}{\sfdefault}{bx}{n}

\graphicspath{{figure/}}

\title{Long-Horizon Scaling: How Model Capabilities Shape the Returns to Computation}
\author{
Haoyu Zheng\textsuperscript{1,2,$\dagger$}, Zhengyu Chen\textsuperscript{1,*},
Huaisheng Zhu\textsuperscript{1}, Ruishan Fang\textsuperscript{1,3,$\dagger$} \\
\textbf{Teng Xiao\textsuperscript{4,5}, Yiwei Li\textsuperscript{1},
Jingang Wang\textsuperscript{1}, Wenqiao Zhang\textsuperscript{2,*}} \\[4pt]
\textsuperscript{1}Meituan LongCat Team \quad \textsuperscript{2}Zhejiang University \\
\textsuperscript{3}Westlake University \quad \textsuperscript{4}Allen Institute for AI \\
\textsuperscript{5}University of Washington
}
\renewcommand{\shorttitle}{Long-Horizon Scaling: How Model Capabilities Shape the Returns to Computation}
\renewcommand{\headeright}{}

\begin{document}
\maketitle
\begingroup
\renewcommand{\thefootnote}{\fnsymbol{footnote}}
\footnotetext[2]{Work done during an internship at Meituan.}
\footnotetext[1]{Corresponding authors.}
\endgroup

\begin{abstract}
Long-horizon agents improve solutions through sustained interaction, execution, and task feedback. Scaling studies relate performance to resources and capabilities, yet how existing capabilities shape returns to extended interaction remains less understood. To address this gap, we analyze AutoLab and EdgeBench, two long-horizon benchmarks. We find that starting performance and subsequent growth are associated with different capabilities: within a task category, similar early scores can precede different later gains. To formalize this finding, we model capability--time scaling with category-specific logistic power laws shared across models. Fitted to early trajectories, these curves extrapolate the observed models' category-average scores to later computation. However, rising average scores mask narrowing improvement opportunities: later gains concentrate among fewer improving models. High final scores and continued improvement also have distinct capability profiles. Predicted mean gains estimate each model's fraction of improving tasks; averaging these estimates forecasts the average share of improving models. These uneven returns motivate deciding whether a specific run should continue. We therefore derive a continuation policy to save time and compute with limited score loss. The policy conditions growth predictions on the run's observed progress and weighs immediate and delayed gains against computation costs. In replay with training and price calibration based on other models' histories, the policy saves roughly one-third of full-run time, with relative score losses of $2.4\%$ on AutoLab individual runs and $3.3\%$ on EdgeBench published mean curves. Our repository is available at \url{https://github.com/Chihaya-Anon-chan/long-horizon-scaling}.
\end{abstract}

\section{Introduction}
\label{sec:introduction}

\begin{figure}[!htbp]
\centering
\includegraphics[width=\linewidth]{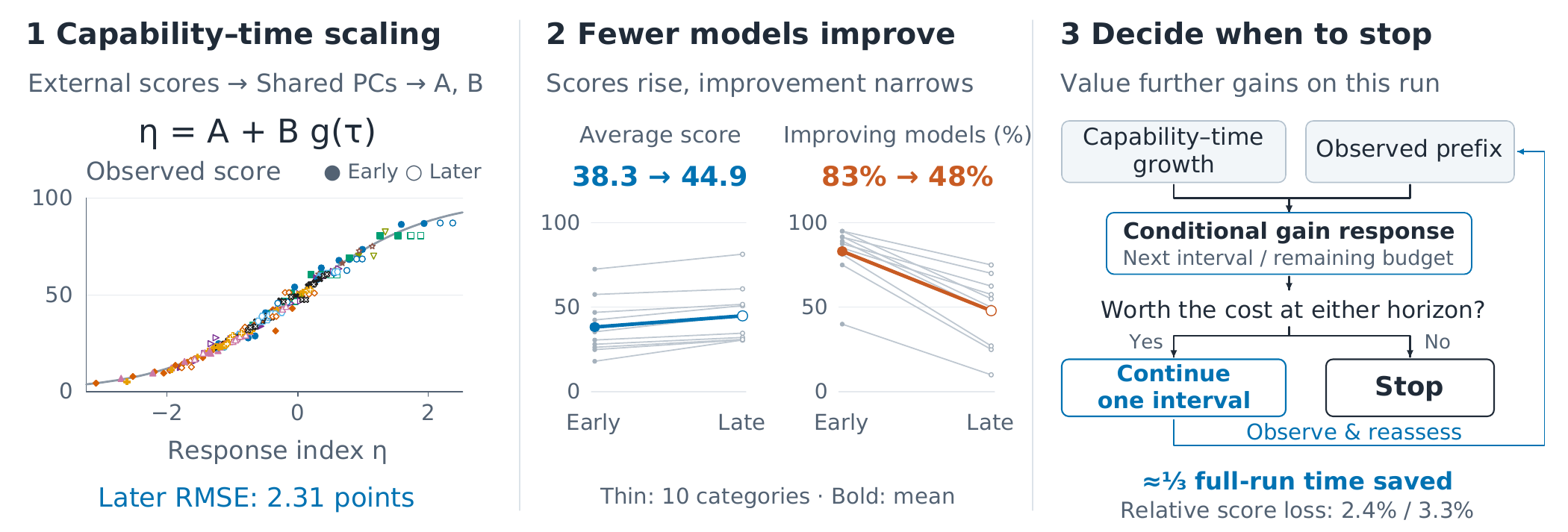}
\caption{From capability--time scaling to adaptive continuation. (1) All $254$ observations versus the fitted index. The gray curve shows the logistic link; filled/open markers denote early/later observations. RMSE is averaged across categories. (2) Thin lines show ten categories; bold lines show their means. Scores use first/last interval endpoints (AutoLab $40/100\%$ budget; EdgeBench $4/12$ h). (3) Growth fitted on other models and the current prefix guide stopping. Reported tradeoffs use prices calibrated on other models and then fixed (AutoLab/EdgeBench; Section~\ref{sec:continuation_results}).}
\label{fig:scaling_overview}
\end{figure}

Language models increasingly undertake research and engineering through sustained interaction, execution, and refinement~\citep{lu2026aiscientist,wijk2025rebench}. Understanding these capabilities requires measuring growth with computation, beyond final scores. Scaling laws frame this question: training scaling relates performance to model size, data, and compute~\citep{kaplan2020scaling,hoffmann2022training}; observational scaling uses shared capability representations~\citep{ruan2024observational,polo2025sloth}; test-time scaling examines inference computation~\citep{snell2025scaling,schaeffer2025monkeys}. Recent work models growth during sustained interaction~\citep{zhu2026edgebench}, but its relation to model capabilities remains less understood. We therefore ask: \emph{How do shared model capabilities shape long-horizon growth and the distribution of gains from further computation?}

Our central finding is that a shared capability space organizes both starting performance and subsequent growth through task-specific combinations (Figure~\ref{fig:scaling_overview}, left). Task trajectories measure progress on particular tasks. We use external benchmarks as a common reference to compare model capabilities across tasks. Their correlated scores support a compact representation with multiple directions of model variation~\citep{ruan2024observational,zeng2026benchpress}. For each task category, two learned sets of weights combine a model's coordinates to describe starting performance and growth over time, respectively. Both sets are shared across models. These combinations parameterize a logistic power law, allowing similar starting scores to precede different later gains. The fitted weights express these growth patterns as external benchmark combinations. This shared representation lets us compare how capabilities relate to performance across task categories and budgets. Fits to early trajectories predict the observed models' later category-average scores.

The growth model above also helps explain why average scores rise while fewer models continue to improve (Figure~\ref{fig:scaling_overview}, middle). On average across tasks in each category, declining participation (the share of improving models) concentrates gains, partly offset by more balanced gains among those that improve. A relation calibrated on early intervals uses predicted mean gains to estimate the fraction of tasks on which each model later improves. Averaging these forecasts across models estimates task-averaged participation. External benchmark associations help interpret these patterns. In systems and software engineering, repository and execution scores correlate more strongly with final scores than with continued improvement, whereas scientific reasoning scores correlate more strongly with continued improvement than with final scores. Gain timing shows that models with fewer late improvements may have already realized most of their observed gains.

This finding---that fewer models improve in later intervals---motivates deciding whether a particular model should continue its current task run. Runs sharing a category-level reference can differ in attained gains and recent stagnation. Conditioning capability--time growth learned from other models on the current run's observed history yields a bounded curve for next-interval and remaining-budget gains (Figure~\ref{fig:scaling_overview}, right). The policy purchases one interval when either forecast justifies its computation cost~\citep{hay2012computations}, then reassesses. The remaining-budget forecast allows the policy to wait for delayed gains even when the next interval alone does not justify its cost.

We examine these growth patterns and evaluate the resulting continuation policy on AutoLab and EdgeBench, whose research and engineering tasks require sustained execution and task feedback~\citep{xu2026autolab,zhu2026edgebench}. Our AutoLab individual-run trajectories and official EdgeBench task--model mean curves support category-level analysis of ten categories, ten model versions, and $71$ tasks. The category analysis tests temporal extrapolation for observed models and forecasts later improvement frequencies. Continuation training and price calibration exclude target models' long-horizon histories; decisions use external profiles and observed prefixes.

This joint analysis yields three connected contributions:
\begin{enumerate}
    \item \textbf{Capability-dependent growth and gain concentration.} Task-specific capability combinations describe starting performance and growth, and help explain shrinking participation. Benchmark profiles distinguish final scores from continued improvement (Sections~\ref{sec:capability_time}--\ref{sec:gain_distribution}).
    \item \textbf{A continuation policy derived from capability--time growth.} Conditioning capability-dependent progress on the current run yields a bounded curve for immediate and delayed gains, guiding further computation through sequential reassessment (Section~\ref{sec:continuation}).
    \item \textbf{Empirical validation across two long-horizon benchmarks.} Early fits predict later scores with mean category $R^2=0.844$ and RMSE $2.31$ points. Mean participation falls $35.2$ percentage points from first to last interval. In replay, the policy saves roughly one-third of full-run time at $2.4\%$ relative score loss on AutoLab and $3.3\%$ on EdgeBench.
\end{enumerate}

\section{Measurement Framework and Experimental Setting}
\label{sec:measurement}
\label{sec:setup}

We study how verified solutions improve over sustained computation. AutoLab covers CUDA kernel optimization, model development, puzzles and challenges, and system optimization~\citep{xu2026autolab}. EdgeBench includes formal mathematics, games, knowledge work, optimization, scientific and machine-learning tasks, and systems and software engineering~\citep{zhu2026edgebench}. Growth describes score trajectories over computation; gains are interval increments, and improvements are positive increments. Checkpoints reveal their timing and distribution.

Let $m$ index a model version and $i$ a concrete task in category $c(i)$. Runs have a predeclared budget $T$ and normalized time $\tau=t/T$. At predeclared checkpoints ending at $\tau_N=1$, trajectories record their best verified score in $[0,1]$; natural termination is absorbing. Write $s_{mi}(\tau)$ for a task--model mean curve. For fixed tasks $\mathcal I_c$ shared by models $\mathcal M_c$, the category response targets
\begin{equation}
    \bar s_{mc}(\tau)=\frac{1}{|\mathcal I_c|}\sum_{i\in\mathcal I_c}s_{mi}(\tau).
    \label{eq:category_mean}
\end{equation}
Our category panels contain $71$ tasks, four AutoLab and six EdgeBench categories, and four or five models per category ($10$ distinct versions overall). Membership depends only on trajectory availability. AutoLab curves average available replicates from the evaluations of \citet{li2026beyondfinalscores}; EdgeBench curves are published task--model means. Equations use $[0,1]$ scores; results report scores and errors on $0$--$100$. At checkpoint $k$, continuation uses score $Y_k$ and observed history $\mathcal F_k$.

Fixed membership keeps temporal comparisons on the same tasks and models within each category. Continuation replay uses all eligible histories under its held-out-model protocol (Section~\ref{sec:continuation_results}); it does not require every task to be shared by the category panel.

\section{Capability--Time Scaling of Long-Horizon Performance}
\label{sec:capability_time}

\begin{figure}[!t]
\centering
\includegraphics[width=\linewidth]{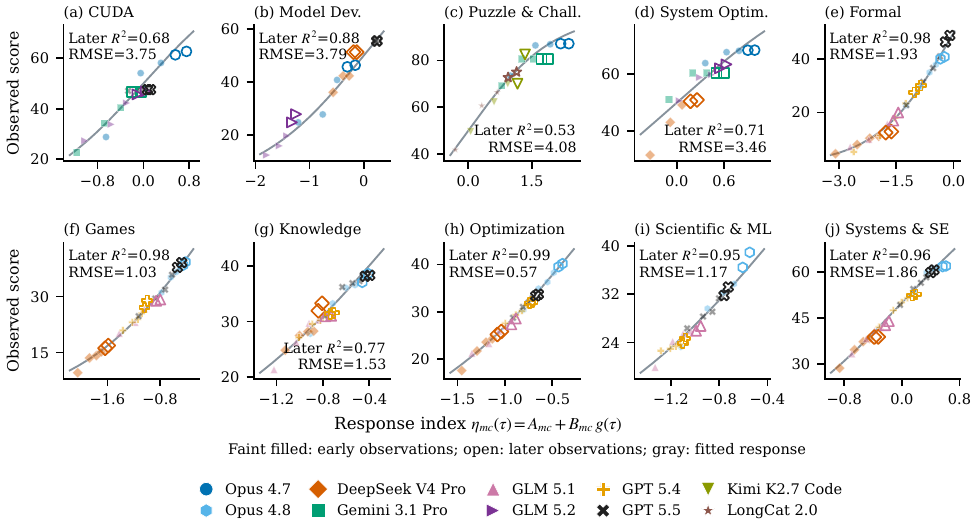}
\caption{Shared capabilities organize category--time responses across all ten categories: AutoLab (a--d), EdgeBench (e--j). Filled/open markers show early/later category means, with model identity fixed by color and shape; all $254$ means are retained. Gray curves show $100\sigma(\eta)$, and annotations report later $R^2$ and RMSE in score points. Section~\ref{sec:temporal_fit} specifies the early/later checkpoints and five-PC fitting protocol; Table~\ref{tab:task_content} lists the category panels. Axis ranges vary by category.}
\label{fig:category_responses}
\end{figure}

\subsection{Shared Capabilities, Category-Specific Growth}
\label{sec:capability_space}
Shared capability coordinates organize the level and timing of progress through category-specific readouts. Let $\mathbf x_m\in\mathbb R^{29}$ be model $m$'s preprocessed profile of $29$ external benchmark scores (Appendix~\ref{sec:method_details}). We estimate a regularized covariance from $82$ reference entries excluding evaluated versions and aliases, and retain five principal components (PCs)~\citep{jolliffe2016pca}:
\begin{equation}
    \mathbf z_m=\mathbf V_5^{\mathsf T}\mathbf x_m\in\mathbb R^5.
    \label{eq:capability_projection}
\end{equation}
The columns of $\mathbf V_5\in\mathbb R^{29\times5}$ are the retained covariance eigenvectors. Each PC combines the external measurements, and the same model coordinates serve all categories. PC1 is a broad positive direction across all 29 benchmarks; PC2 contrasts agent-execution with knowledge and reasoning measurements. Figure~\ref{fig:capability_basis} shows all loadings and model coordinates; external-only missing-data estimation and preprocessing appear in Appendix~\ref{sec:method_details}.

For first checkpoint $0<\tau_0<1$, define normalized logarithmic time
\begin{equation}
    g(\tau)=\frac{\log(\tau/\tau_0)}{\log(1/\tau_0)},
    \qquad g(\tau_0)=0,\quad g(1)=1.
    \label{eq:log_time}
\end{equation}
Logarithmic time measures proportional compute increases from the first checkpoint to the endpoint. Category-specific capability combinations determine early performance and temporal response:
\begin{align}
    A_{mc}&=a_c+\mathbf w_c^{\mathsf T}\mathbf z_m,
    &B_{mc}&=b_c+\mathbf v_c^{\mathsf T}\mathbf z_m,
    \label{eq:capability_readouts}\\
    \eta_{mc}(\tau)&=A_{mc}+B_{mc}g(\tau),
    &\widehat q_{mc}(\tau)&=\sigma\!\left(\eta_{mc}(\tau)\right),
    \label{eq:scaling_response}
\end{align}
Here $\widehat q_{mc}$ predicts the category mean $\bar s_{mc}$, and $\sigma(u)=(1+e^{-u})^{-1}$ maps the response index $\eta$ to $[0,1]$. $A_{mc}$ is the predicted first-checkpoint score logit; $B_{mc}$ is the modeled logit change to the endpoint. Each category shares its intercepts $a_c,b_c$ and weights $\mathbf w_c,\mathbf v_c$ across models; substituting a model's coordinates yields its $A_{mc},B_{mc}$. Separate readouts allow similar starting performance to coexist with different subsequent growth.

The response is a \emph{logistic power law}, since
\begin{equation}
    \frac{\widehat q_{mc}(\tau)}{1-\widehat q_{mc}(\tau)}
    =e^{A_{mc}}\left(\frac{\tau}{\tau_0}\right)^{\kappa_{mc}},
    \qquad \kappa_{mc}=\frac{B_{mc}}{\log(1/\tau_0)}.
    \label{eq:odds_power_law}
\end{equation}
The initial score odds are $e^{A_{mc}}$; doubling computation multiplies them by $2^{\kappa_{mc}}$. The score grows at $\kappa_{mc}\widehat q_{mc}(1-\widehat q_{mc})$ per unit log time, so gains diminish near saturation even with a fixed exponent. A matched comparison with all 29 external inputs favors the logistic over the affine log-time response in eight of ten categories (Appendix~\ref{sec:response_details}).

The five-PC projection yields benchmark combinations that depend on category and time:
\begin{equation}
    \eta_{mc}(\tau)=a_c+b_cg(\tau)
       +\left\{\mathbf V_5\bigl[\mathbf w_c+g(\tau)\mathbf v_c\bigr]\right\}^{\mathsf T}
        \mathbf x_m.
    \label{eq:benchmark_readout}
\end{equation}
The effective weights $\mathbf V_5[\mathbf w_c+g(\tau)\mathbf v_c]$ combine all $29$ measurements through the shared five-PC space. Each category's $\mathbf w_c$ sets the initial weights; $\mathbf v_c$ determines how they change with time.

\begin{figure}[!htbp]
\centering
\includegraphics[width=\linewidth]{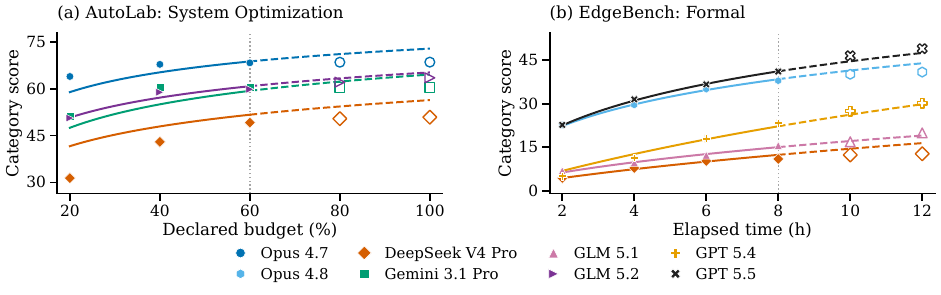}
\caption{Observed trajectories show early gain realization and sustained growth. Fits from Figure~\ref{fig:category_responses}; all 50 model--checkpoint means: (a) 12 tasks, four models; (b) eight tasks, five models. Filled/solid: early observations/fits; open/dashed: later observations/predictions. Dotted: fitting cutoffs.}
\label{fig:category_time_trajectories}
\end{figure}

\subsection{Forecasting Later Category Performance}
\label{sec:temporal_fit}
Within each category, we fit coefficients shared across models by regularized squared error on early scores: AutoLab $20\%,40\%,60\%$ budget and EdgeBench $2,4,6,8$ hours. Later evaluation uses $80\%,100\%$ and $10,12$ hours, respectively. At a fixed dimension, coefficients and regularization use early checkpoints only; Appendix~\ref{sec:method_details} details dimension selection and comparisons.

Across all $90$ later model--checkpoint means, mean category $R^2$ is $0.844$, weighting categories equally (Figure~\ref{fig:category_responses}). These forecasts extend the observed models' early trajectories to later computation. Category RMSE ranges from $0.57$ to $4.08$ points, with five categories exceeding $R^2=0.95$.

On the same later points, retaining the last early score gives mean RMSE $3.10$; fitting an independent logistic time curve to each model gives $2.52$; replacing PCs with regularized model indicators gives $2.66$. Five PCs attain $2.31$, improving on independent curves in seven categories (Appendix~\ref{sec:temporal_controls}).

Figure~\ref{fig:category_time_trajectories} places these responses on budget and time axes to expose when gains occur. System Optimization shows substantial early gains followed by flatter observed trajectories; in Formal, Opus~4.8 and GPT~5.5 start at similar scores but diverge as computation proceeds.

\FloatBarrier

\section{Shrinking Improvement Opportunities and Concentrated Gains}
\label{sec:gain_distribution}

For each interval, \emph{participation} measures the fraction of models improving on a task; \emph{improvement frequency} measures the fraction of tasks improved by a model.

\subsection{Later Gains Reach Fewer Models}
\label{sec:gain_results}
The growth curves describe how much each model improves on average across tasks. To understand who benefits from further computation, we examine how widely models share gains on each concrete task. Let $\Delta_{mi,k}=s_{mi}(\tau_{k+1})-s_{mi}(\tau_k)\ge0$ and $M_c=|\mathcal M_c|$. Define participation $P$ and effective gain coverage $C$~\citep{hill1973diversity} by
\begin{equation}
    P_{i,k}=\frac{1}{M_c}\sum_m\mathbf1\{\Delta_{mi,k}>0\},
    \qquad C_{i,k}=\frac{(\sum_m\Delta_{mi,k})^2}{M_c\sum_m\Delta_{mi,k}^2}.
    \label{eq:participation_coverage}
\end{equation}
Sums range over $\mathcal M_c$. Coverage is one for equal positive gains across all models and zero when none improves. For positive gains, their population coefficient of variation gives the exact identity
\begin{equation}
    C_{i,k}=P_{i,k}U_{i,k},\qquad
    U_{i,k}=\frac{1}{1+\mathrm{CV}_{+,i,k}^2},\qquad C_{i,k}\le P_{i,k}.
    \label{eq:coverage_identity}
\end{equation}
Participation measures how widely improvement occurs; $U$ measures how evenly positive gains are distributed among improving models. Their product separates these two sources of concentration and is unchanged by a common rescaling of gain magnitudes.

Category-mean scores rise at every recorded interval in all ten categories. Yet across the $71$ fixed tasks, task-averaged participation falls in every category from the first to the last interval, by $35.2$ percentage points on a ten-category average (Figure~\ref{fig:growth_opportunity_profiles}(a)). Positive gains become more balanced in seven categories. For tasks with positive gains in both intervals, fewer improving models drive the coverage decline; more evenly shared gains partly offset it. Requiring gains above $0.1$ points preserves the decline in all ten categories. Separately, expanding the positive-gain analysis to all available complete task--model curves yields declines in all ten categories, averaging $33.5$ percentage points (Appendix~\ref{sec:coverage_details}).

\subsection{From Mean Gains to Improvement Frequency}
For a given model, task category, and interval, mean gain equals the fraction of tasks that improve multiplied by their mean gain, provided at least one task improves. A smaller mean can reflect fewer improving tasks, smaller improvements, or both. The growth curves predict mean gains for each model and task category; we use those gains to estimate the fraction of tasks that improve.

Let $\widehat\mu_{mc,k}=\widehat q_{mc}(\tau_{k+1})-\widehat q_{mc}(\tau_k)$ be the predicted category-mean increment. Using the positive increments of the fitted curves, we model the fraction of tasks with positive gain as
\begin{equation}
    \widehat\pi_{mc,k}
    =\frac{1}{1+(\mu_c^\star/\widehat\mu_{mc,k})^\gamma},
    \qquad \mu_c^\star>0.
    \label{eq:opportunity_response}
\end{equation}
The category scales $\mu_c^\star$ and shared exponent $\gamma=1.75$ are fitted to early improvement frequencies. At $\widehat\mu=\mu_c^\star$, half the tasks are predicted to improve.

Figure~\ref{fig:growth_opportunity_profiles}(b) groups later model--category--interval observations with similar predictions. The grouped predictions track the observed fractions of improving tasks.

In the fitted low-gain regime, improvement frequency falls proportionally faster than mean gain. When $\widehat\mu\ll\mu_c^\star$, $\widehat\pi\approx(\widehat\mu/\mu_c^\star)^{1.75}$: halving the predicted mean increment leaves about $30\%$ of the predicted improvement frequency. Mean scores can thus rise while improvement reaches a shrinking share of tasks. Starting level $A_{mc}$ and temporal response $B_{mc}$ jointly set the interval gain: near saturation, further logit growth yields smaller score increments. The scale $\mu_c^\star$ determines how these increments translate into improvement frequency for that task category. These capability-dependent growth histories place models in the low-opportunity regime at different times.

Within a category, averaging improvement frequency over models equals averaging participation over tasks. The mean predicted frequency therefore estimates task-averaged participation. Fewer improving models reduce the upper bound on gain coverage, $C\le P$; the balance of gains among them determines $U$ in $C=PU$. Appendix~\ref{sec:opportunity_calibration} evaluates later calibration and compares forecast inputs under the same early-to-late split.

\begin{figure}[!htbp]
\centering
\includegraphics[width=\linewidth]{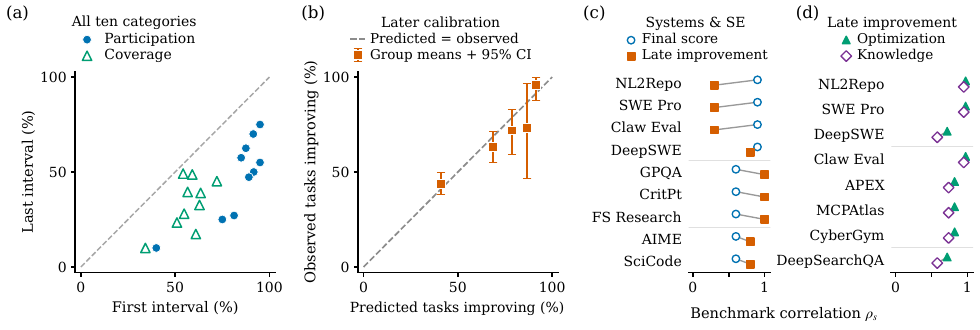}
\caption{Fewer models improve later, with category-specific capability profiles. (a) First/last-interval participation and gain coverage across ten categories. (b) Task-weighted group means of predicted and observed improvement frequencies for all $90$ later observations. Groups use early prediction quantiles; bars show $95\%$ whole-task bootstrap intervals. Dashed lines indicate equality. (c) Systems \& SE: benchmark correlations with final score and terminal improvement frequency. (d) Optimization's eight benchmarks with $\rho_s\ge0.70$ for terminal frequency, compared with Knowledge. Correlations use five models per category; terminal frequency averages the last two intervals. Lines pair the same benchmark. Details and measurements: Appendix~\ref{sec:coverage_details}.}
\label{fig:growth_opportunity_profiles}
\end{figure}

\subsection{Capability Profiles of Long-Horizon Progress}
\label{sec:task_signatures}

We next examine which external benchmarks are associated with high final scores, continued improvement, and earlier gains. We compare all $29$ external measurements with final scores and improvement frequencies using within-category model rank correlations. These associations describe capability profiles rather than isolated benchmark effects. Terminal improvement frequency averages the fraction of tasks improving in the last two intervals. Gain timing distinguishes persistent progress from earlier gain realization (sources: Table~\ref{tab:external_sources}; full profiles: Appendix~\ref{sec:task_signature_details}).

\textbf{Attainment and continued improvement have distinct capability profiles.} Across all $29$ benchmarks, the average signed rank correlation across categories is more positive for final scores than for terminal improvement frequency. This difference persists when removing each model across categories and when varying the terminal window and improvement threshold (Appendix~\ref{sec:task_signature_details}). In Systems \& SE, repository engineering and execution measurements---NL2Repo, SWE-bench Pro~\citep{ding2025nl2repo,deng2025swepro}, and Claw Eval---track final scores more closely than terminal improvement frequency (Figure~\ref{fig:growth_opportunity_profiles}(c)). Scientific reasoning (GPQA, CritPt, FrontierScience Research) tracks continued improvement more closely than final scores; mathematics and coding (AIME, SciCode, DeepSWE) also track improvement~\citep{rein2024gpqa,tian2024scicode}.

\textbf{Engineering and execution recur across task categories.} Optimization exhibits a composite profile spanning repository engineering (NL2Repo, SWE-bench Pro, DeepSWE), execution and tool use (Claw Eval, APEX Agents, MCPAtlas, CyberGym), and information seeking (DeepSearchQA). All eight align with continued improvement in executable solvers; seven retain positive associations across nine choices of time window and improvement threshold. Six also track continued progress in Knowledge, where three of four tasks deliver domain-specific systems (Figure~\ref{fig:growth_opportunity_profiles}(d)). The shared profile reflects implementation and revision work across different subject matter.

\textbf{Fewer late improvements can accompany earlier gain realization.} In AutoLab System Optimization, AA-LCR, GPQA, SciCode, SimpleQA, MMLU-Pro, and MathArena associate with both fewer terminal improvements and earlier gains. By $60\%$ budget, Opus~4.7, Gemini~3.1~Pro, and DeepSeek~V4~Pro have realized $91$--$100\%$ of their observed $20$--$100\%$ gains, versus $71.5\%$ for GLM~5.2 (Figure~\ref{fig:category_time_trajectories}a). Full-trajectory timing preserves this profile (Appendix~\ref{sec:gain_timing}).

\section{From Scaling to Adaptive Continuation}
\label{sec:continuation}

Continued category growth can coexist with little progress on the current task. We evaluate further computation using a response to capability--time growth conditioned on the run's observed state.

\subsection{Conditioning Growth on the Current Run}
\label{sec:conditional_curve}
At each checkpoint, one bounded curve predicts next-interval and remaining-budget gains. Capability--time growth defines reference progress; attained score, recent improvement, and stagnation condition how much the current run gains from that progress.

For each held-out model, we refit capability--time growth on other models' completed histories, with fixed external PCs and a positive time coefficient. Write $r_{mc}(\tau)$ for this reference and $H_k=1-Y_k$ for the remaining score space. Suppressing $m,c$, we use a Weibull response~\citep{weibull1951distribution}:
\begin{equation}
    \begin{aligned}
    D_k(\tau)&=\log\frac{1-r(\tau_k)}{1-r(\tau)},
    &&\tau_k\le\tau\le1,\\
    \widehat G_k(\tau)&=H_k\left[1-\exp\{-\alpha_kD_k(\tau)^{\beta_k}\}\right],
    &&\alpha_k,\beta_k>0.
    \end{aligned}
    \label{eq:conditional_gain}
\end{equation}
$D_k$ measures the reference's reduction in remaining score space on a log scale. At $\alpha_k=\beta_k=1$, the curve transfers its fractional reduction, $1-[1-r(\tau)]/[1-r(\tau_k)]$, to the current run's remaining space. Intensity $\alpha_k$ adjusts gain magnitude; shape $\beta_k$ adjusts its dependence on reference progress. The curve starts at zero, increases with horizon, and stays bounded by $H_k$.

The positive coefficients depend on seven features of the observed prefix, with task effects learned from other models' histories. Equal time intervals can represent different reference progress for different model--category pairs. Runs of the same model on the same task share a reference but receive different forecasts as their observed progress diverges.

We fit this single curve jointly to next-interval and remaining-budget gains, including zero-gain histories. All policy fitting excludes the target model's long-horizon histories; each source model is also excluded when constructing its training references. Source histories on the same task are allowed, while the target run contributes only its observed prefix. Appendix~\ref{sec:source_estimation} gives the feature definitions, regression, and training objective.

\subsection{Purchase One Interval and Reassess}
\label{sec:decision_rule}
For $h\in\{\mathrm n,\mathrm e\}$, set $\tau_{k,\mathrm n}=\tau_{k+1}$ and $\tau_{k,\mathrm e}=1$. The conditional curve gives ordered gains:
\begin{equation}
    \widehat g_{k,h}=\widehat G_k(\tau_{k,h}),
    \qquad 0\le\widehat g_{k,\mathrm n}\le\widehat g_{k,\mathrm e}\le H_k.
    \label{eq:gain_readouts}
\end{equation}
Let $\lambda\ge0$ price each unit of declared budget in score gain. Subtracting computation costs from predicted gains gives the net values~\citep{hay2012computations}:
\begin{equation}
    \widehat Q_{k,h}=\widehat g_{k,h}-\lambda c_{k,h},
    \qquad c_{k,h}=\tau_{k,h}-\tau_k.
    \label{eq:net_values}
\end{equation}
If $\max_h\widehat Q_{k,h}>0$, purchase the next interval and reassess. The remaining-budget forecast allows the policy to continue even when the next interval alone is not worth its cost. Full-score and naturally terminated runs stop. We use source histories to select a nonnegative price or the run-to-end option and keep that choice fixed for the target model. Realized duration is used only for cost evaluation.

Let $Q_{k,h}$ denote equation~\ref{eq:net_values} evaluated with true conditional mean gains given $\mathcal F_k$. With these means at every visited checkpoint, the expected net value of reassessment satisfies
\begin{equation}
    V_k^{\mathrm{seq}}\ge\max(0,Q_{k,\mathrm n},Q_{k,\mathrm e}).
    \label{eq:exact_value_bound}
\end{equation}
Here $V_k^{\mathrm{seq}}$ is the conditional expectation of subsequent gain minus priced declared cost. The plans---stop, one interval, and endpoint---use full declared costs, with $\lambda(1-\tau_k)$ for the endpoint. The bound assumes a finite declared grid, additive costs, history-based actions, and stopping that truncates the same potential trajectory. Proposition~\ref{prop:sequential_value} gives the proof and bounds losses from gain-sign errors.

\subsection{Cost--Performance Tradeoffs}
\label{sec:continuation_results}
We replay $492$ individual AutoLab seed runs over seven target models and $252$ published EdgeBench mean curves over five. Each target provides its external profile and current prefix, with its long-horizon histories held out from training.

Fixed-budget stopping uses a preset cutoff independent of observed progress. Time-based patience stops after a threshold fraction of the declared horizon since the last improvement, while recent-gain rules threshold the gain rate over one or two observed intervals.

\begin{figure}[!htbp]
\centering
\includegraphics[width=\linewidth]{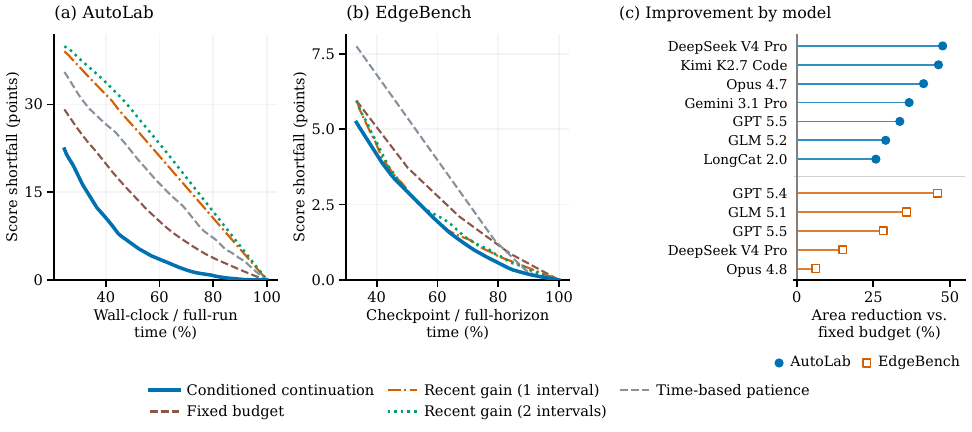}
\caption{Conditioned continuation improves cost--loss tradeoffs. (a--b) Frontiers averaged over replicates then models. (c) Reductions in mean score loss versus fixed-budget stopping; all seven AutoLab and five EdgeBench evaluations improve. Comparisons use the same cost range across policies for each model and replicate; (a--b) display the range shared across each suite's evaluations.}
\label{fig:continuation_frontiers}
\end{figure}

In Figure~\ref{fig:continuation_frontiers}, cost is charged time including initial observation, divided by full-run time: wall-clock time for AutoLab, checkpoint hours for EdgeBench. Loss is the score shortfall relative to completing the same trajectory. We average loss over the policies' shared cost range within each model--replicate, then over replicates and models (normalized frontier area; Appendix~\ref{sec:operating_protocol}).

Conditioned continuation has the lowest mean loss over these cost ranges in both suites. Relative to fixed-budget stopping, mean loss falls from $14.38$ to $8.83$ points on AutoLab and from $2.44$ to $1.77$ on EdgeBench, reductions of $38.6\%$ and $27.4\%$ (Table~\ref{tab:continuation_results}). Mean stopping loss falls for every model (Figure~\ref{fig:continuation_frontiers}c). The available EdgeBench trajectories do not retain run-to-run variation in improvement timing within a task--model pair. This limits the run-specific information available for adaptation and may narrow the margin over fixed-budget stopping.

Saving half the time retains $90.5\%$ of the full-run score on AutoLab and $91.8\%$ on EdgeBench, versus $77.1\%$ and $89.5\%$ for fixed-budget stopping (Table~\ref{tab:matched_tradeoffs}). At $95\%$ score retention, time savings are $38.1\%$ on AutoLab and $38.4\%$ on EdgeBench, versus $17.2\%$ and $28.5\%$ for fixed-budget stopping.

\begin{table}[!htb]
\centering
\caption{Matched operating points, interpolated from the mean frontiers in Figure~\ref{fig:continuation_frontiers}. Score retention is relative to the mean full-run score. Ours denotes conditioned continuation; higher is better. Table~\ref{tab:matched_tradeoffs_full} compares all baselines.}
\label{tab:matched_tradeoffs}
\small
\begin{tabular*}{\linewidth}{@{\extracolsep{\fill}}llrrrr@{}}
\toprule
 & & \multicolumn{2}{c}{\textbf{AutoLab}} & \multicolumn{2}{c}{\textbf{EdgeBench}}\\
\cmidrule(lr){3-4}\cmidrule(lr){5-6}
Matched condition & Reported outcome & Fixed budget & Ours & Fixed budget & Ours\\
\midrule
Save $30\%$ time & Score retained (\%) & $89.9$ & $\mathbf{97.3}$ & $94.7$ & $\mathbf{96.8}$\\
Save $50\%$ time & Score retained (\%) & $77.1$ & $\mathbf{90.5}$ & $89.5$ & $\mathbf{91.8}$\\
\midrule
Retain $95\%$ score & Time saved (\%) & $17.2$ & $\mathbf{38.1}$ & $28.5$ & $\mathbf{38.4}$\\
Retain $90\%$ score & Time saved (\%) & $29.7$ & $\mathbf{50.9}$ & $48.2$ & $\mathbf{55.9}$\\
\bottomrule
\end{tabular*}
\end{table}

To choose an operating point before observing the target run, we calibrate the price on other models' runs to meet a target mean fraction of declared budget. We then fix that price for the target model. At these operating points, continuation saves $32.8\%$ of full-run time on AutoLab with a loss of $1.52$ points ($2.4\%$ relative to full-run scores), and $30.9\%$ on EdgeBench with a loss of $1.19$ points ($3.3\%$). Table~\ref{tab:frozen_operating} reports the source targets and complete operating points.

\section{Related Work}
\label{sec:related_work}

\paragraph{Capability-based scaling.}
Downstream scaling relates benchmark performance to pretraining loss or perplexity~\citep{gadre2025downstream,du2024loss}. Observational scaling and Sloth instead organize performance through latent capabilities, accounting for family-specific compute efficiencies~\citep{ruan2024observational,polo2025sloth}. Shared performance structure also enables efficient evaluation: tinyBenchmarks estimates scores from selected examples~\citep{polo2024tinybenchmarks}, while BenchPress predicts missing benchmark scores from low-rank public score matrices~\citep{zeng2026benchpress}. We use external capability coordinates to describe both starting performance and within-task temporal growth, with separate readouts for each task category.

\paragraph{Long-horizon performance and compute returns.}
Long-horizon evaluations measure agent reliability in units of human task duration~\citep{kwa2025horizon} and compare performance across computation budgets. RE-Bench compares best-of-$k$ agents with human experts across budgets~\citep{wijk2025rebench}; AutoLab evaluates sustained research and engineering~\citep{xu2026autolab}. On AutoLab, \citet{li2026beyondfinalscores} study process behavior, experience reuse, and harness effects. EdgeBench fits benchmark- and task-family growth curves and tests temporal extrapolation~\citep{zhu2026edgebench}. Test-time scaling increases computation through repeated sampling~\citep{brown2024monkeys} or longer reasoning~\citep{muennighoff2025s1}; allocation depends on problem difficulty~\citep{snell2025scaling}. Heterogeneous success probabilities explain sampling power laws~\citep{schaeffer2025monkeys}, while transition models describe iterative self-correction~\citep{yang2025selfcorrection}. We study capability-dependent growth, gain concentration, and continuation across tasks and models. 

\paragraph{Adaptive stopping and computation value.}
Rational metareasoning compares expected decision-quality gains with compute costs~\citep{hay2012computations}. AgentStop predicts unsuccessful episodes using token probabilities and trace features~\citep{pham2026agentstop}, while \citet{ruan2026earlyabort} use a recall-controlled cascade of hidden-state probes. Our policy prices next-interval and remaining-budget gains from a bounded, prefix-conditioned response to capability-dependent progress.

\section{Discussion and Conclusion}
\label{sec:conclusion}

Model capabilities organize the level and timing of long-horizon progress. Average scores rise across the ten categories while participation falls from the first to last interval. Shrinking participation concentrates gains; capability-dependent growth helps explain each model's later improvement frequency. Evaluations should report performance, gain timing, and the breadth of improvement.

Conditioning category-level growth on observed progress forecasts immediate and delayed gains for a specific run. The resulting policy improves time--score tradeoffs in replay on AutoLab individual runs and EdgeBench published mean curves. Broader model coverage and individual EdgeBench trajectories would extend this evaluation. Capability-based scaling can therefore describe long-horizon progress and guide decisions about the value of further computation.

\FloatBarrier
\clearpage
\section*{Reproducibility Statement}

Appendices~\ref{sec:method_details}--\ref{sec:continuation_fitting} describe the data preprocessing, evaluation splits, fitting procedures, hyperparameters, and continuation replay protocol, and provide the mathematical proofs. The EdgeBench analyses use officially released task--model mean checkpoint scores. The companion repository is being prepared with sourced external benchmark inputs, code for the external capability representation and category-level response, and an index of already-public trajectories. Our own AutoLab individual-run trajectories and the inputs needed to replay the run-level continuation results remain withheld pending redaction. Consequently, the paper's trajectory-based results cannot yet be reproduced from the planned public artifact.

\bibliographystyle{unsrtnat}
\bibliography{references}

\clearpage
\appendix
\section{Representation and Fitting Details}
\label{sec:method_details}

\subsection{External Measurements and the Shared PCs}

External scores are assembled from BenchPress~\citep{zeng2026benchpress} and supplementary published benchmark results. Table~\ref{tab:external_sources} lists the 29 measurements and links their benchmark papers or official releases. External benchmark scores are put on a $[0,1]$ scale. Let $\psi$ denote the logit transformation after clipping the score to $[0.005,0.995]$. For benchmark $j\in\{1,\ldots,J\}$ with $J=29$, let $\bar\psi_j$ and $s_j$ be the reference mean and standard deviation of this transformed score. Let $\omega_j$ be the inverse square root of the number of measurements in its benchmark family. An observed entry $b_{mj}$ produces
\begin{equation}
    x_{mj}=\frac{\omega_j}{\sqrt{\sum_{\ell=1}^{J}\omega_\ell^2}}
       \frac{\psi(b_{mj})-\bar\psi_j}{s_j}.
    \label{eq:external_preprocessing}
\end{equation}
Family weighting prevents several related benchmark variants from receiving proportionally more total squared weight simply because more variants are available. FrontierScience and HMMT each have two measurements; the SWE family has three; other measurements receive unit family weight.

The $82$ reference entries provide $901$ observed cells out of $82\times29$. For the ten target model versions, $248$ of $290$ cells use published scores and $42$ use frozen external-only estimates. The missing-entry estimator uses a rank-two external fit and regularized projection with coefficient $0.1$; observed entries are preserved exactly. This estimator supplies missing values, whereas the five downstream PCs come from a separately estimated full-rank covariance. The covariance update uses conditional first and second moments for missing entries, following the expectation--maximization principle~\citep{dempster1977em}, and linear shrinkage toward the identity~\citep{ledoit2004covariance}. Its shrinkage coefficient $0.1$ is selected by hiding observed external entries in reference-model validation splits. The covariance is then family-weighted before extracting its leading eigenvectors. The shared PC coordinates are not whitened. All external preprocessing and target completion are fixed before long-horizon response fitting.

\paragraph{Dimensionality and interpretation.}
The fixed-dimension sweep in Table~\ref{tab:pc_dimensions} informs our five-PC choice through compactness and later-checkpoint errors. At each dimension, response coefficients, regularization, and temporal structure use only early checkpoints. Selecting dimension from last-early-checkpoint errors alone, over one to 29 PCs, gives either a common dimension of 19 or category-specific dimensions; the table reports both protocols. Individual PC names interpret the fitted external loadings: PC1 is a broad positive direction; subsequent PCs describe relative benchmark patterns. PC signs are conventional, with contrast loadings expressing relative performance across measurements.

\begin{table}[ht]
\centering
\caption{Dimension sweep and early-only dimension selection on the same temporal panels. Metrics average the ten categories equally; RMSE is in score points.}
\label{tab:pc_dimensions}
\begin{tabular}{lrrrrr}
\toprule
Number of PCs & 1 & 3 & 5 & 7 & 29\\
\midrule
Mean later-score RMSE & 4.861 & 3.180 & 2.315 & 2.288 & 2.504\\
\bottomrule
\end{tabular}
\par\medskip
\begin{tabular}{lrrr}
\toprule
Early-only dimension selection & Selected PCs & RMSE & $R^2$\\
\midrule
Common across categories & 19 & 2.486 & 0.824\\
Category-specific & 2--23 & 2.434 & 0.830\\
\bottomrule
\end{tabular}
\end{table}

\subsection{Response on the Score Scale}
\label{sec:response_details}

Differentiating the logistic power-law response with respect to log time gives
\begin{equation}
    \frac{\partial\widehat q_{mc}(\tau)}{\partial\log\tau}
    =\kappa_{mc}\widehat q_{mc}(\tau)
      \bigl[1-\widehat q_{mc}(\tau)\bigr].
    \label{eq:response_derivative}
\end{equation}
The same fitted logit growth coefficient yields different score changes depending on the current performance level, with changes diminishing in magnitude near saturation. Equation~\ref{eq:response_derivative} translates the temporal coefficient estimated from category trajectories into score gains along the fitted curve.

\paragraph{Matched response-family comparison.}
We compare the logistic response with its affine counterpart, $\widehat q=A+Bg(\tau)$, using the complete 29-dimensional external profile in both cases. The task and model panels, early checkpoints, native-score squared-error objective, and $90$ later evaluation means match the five-PC analysis. Each family selects its regularization and shared-versus-capability-dependent temporal coefficient on early checkpoints, then refits on the full prefix. Table~\ref{tab:response_links} reports all ten categories. The logistic family reduces mean category RMSE from $2.83$ to $2.50$ points ($11.4\%$), improving eight categories. This comparison evaluates the response family with the external input space held fixed; the five-PC dimension comparison appears in Table~\ref{tab:pc_dimensions}.

\begin{table}[htbp]
\centering
\caption{Matched link-family comparison using all $29$ external inputs. Both families select regularization and slope structure on early checkpoints and predict the same $90$ later means. Errors are RMSE in score points.}
\label{tab:response_links}
\begin{tabular}{lrr}
\toprule
Category & Affine in log time & Logistic power law\\
\midrule
CUDA & 4.42 & 4.02\\
Model Development & 5.68 & 5.73\\
Puzzle \& Challenge & 5.86 & 4.28\\
System Optimization & 3.91 & 3.48\\
Formal & 2.03 & 1.53\\
Games & 1.09 & 1.07\\
Knowledge & 1.65 & 1.57\\
Optimization & 0.41 & 0.54\\
Scientific \& ML & 1.33 & 0.97\\
Systems \& SE & 1.87 & 1.82\\
\midrule
Macro average & 2.83 & 2.50\\
\bottomrule
\end{tabular}
\end{table}

\paragraph{From PC readouts to trajectories.}
Figure~\ref{fig:capability_growth_bridge} connects external loadings, category readouts, and growth in actual time. In the regularized Formal fit, PC2 enters the early-performance readout positively and the temporal readout negatively: its association with the starting level differs from its association with subsequent change. System Optimization illustrates how much of the observed gain occurs before the later evaluation checkpoints (Section~\ref{sec:gain_timing}).

\begin{figure}[htbp]
\centering
\includegraphics[width=\linewidth]{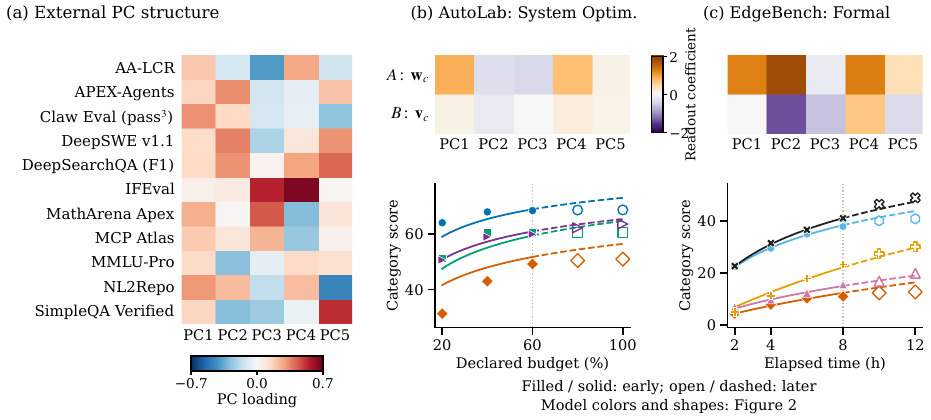}
\caption{From external capabilities to time responses. (a) The union of the two strongest positive and negative loadings per PC shows $11$ benchmarks; PC1 has only positive loadings. All $29$ benchmarks enter the representation; full loadings appear in Figure~\ref{fig:capability_basis}. (b--c) Illustrative categories: top, fitted PC coefficients $\mathbf w_c,\mathbf v_c$ for $A,B$ (intercepts omitted); bottom, all $50$ observed category-mean scores across four/five models on $12/8$ fixed tasks and their fitted curves. Filled/solid denotes early observations/fits; open/dashed denotes later observations/predictions. Both use the five-PC response of Figure~\ref{fig:category_responses}. Loadings and readout coefficients use separate color scales.}
\label{fig:capability_growth_bridge}
\end{figure}

\subsection{Category Response Estimation and Metrics}

\paragraph{Fixed trajectory panels.}
We first retain task--model mean curves with observed scores at every analysis checkpoint in Table~\ref{tab:temporal_protocol}. Within each category, we enumerate model subsets of size at least four and intersect their available tasks. The primary panel maximizes the number of observed task--model pairs; ties favor more tasks, then lexicographic task and model identifiers. Selection uses availability, not score values. This yields the fixed panels in Table~\ref{tab:task_content}, shared by the temporal forecasts and gain-distribution analysis. AutoLab means average available replicates; EdgeBench supplies task--model means. The same tasks and models enter every checkpoint within a category. Continuation uses the separately defined replay population below (Appendix~\ref{sec:operating_protocol}).

\begin{table}[ht]
\centering
\caption{Temporal prediction protocol. AutoLab time is a fraction of declared budget; EdgeBench time is elapsed hours.}
\label{tab:temporal_protocol}
\begin{tabular}{lll}
\toprule
Benchmark & Early coefficient fit & Later evaluation\\
\midrule
AutoLab & $20\%,40\%,60\%$ & $80\%,100\%$\\
EdgeBench & $2,4,6,8$ h & $10,12$ h\\
\bottomrule
\end{tabular}
\end{table}

For category $c$, let $\mathcal E_c$ index its early checkpoints and $n_c=|\mathcal M_c||\mathcal E_c|$. The response parameters are $\theta_c=(a_c,b_c,\mathbf w_c,\mathbf v_c)$. Their fitting objective is
\begin{equation}
    \mathcal L_c(\theta_c)
    =\frac{1}{n_c}\sum_{m\in\mathcal M_c}\sum_{k\in\mathcal E_c}
       \bigl[\widehat q_{mc}(\tau_k)-\bar s_{mc}(\tau_k)\bigr]^2
       +\rho_c\bigl(\|\mathbf w_c\|_2^2+\|\mathbf v_c\|_2^2\bigr)
       +10^{-9}a_c^2+10^{-7}b_c^2.
    \label{eq:temporal_objective}
\end{equation}
The small intercept penalties stabilize the numerical fit. We compare the full response with its restriction $\mathbf v_c=\mathbf0$, and use $\rho_c\in\{10^{-5},10^{-4},10^{-3},10^{-2},10^{-1}\}$. Each candidate is fitted before the last early checkpoint and scored at that checkpoint; ties prefer the restricted response and then stronger regularization. The selected specification is refitted on all early points. The temporal fit leaves $B_{mc}$ unconstrained; all increments used in the opportunity analysis are positive in the selected fits. The continuation reference instead imposes a positive slope by construction.

For the later model--checkpoint set $\mathcal T_c$, write $y_{mk}=\bar s_{mc}(\tau_k)$ and $\bar y_c=|\mathcal T_c|^{-1}\sum_{(m,k)\in\mathcal T_c}y_{mk}$. We report
\begin{align}
    \mathrm{MSE}_c&=\frac{1}{|\mathcal T_c|}\sum_{(m,k)\in\mathcal T_c}
                 [\widehat q_{mc}(\tau_k)-y_{mk}]^2,
    &\mathrm{RMSE}_c&=100\sqrt{\mathrm{MSE}_c},\\
    R_c^2&=1-\frac{\sum_{(m,k)\in\mathcal T_c}[\widehat q_{mc}(\tau_k)-y_{mk}]^2}
                       {\sum_{(m,k)\in\mathcal T_c}(y_{mk}-\bar y_c)^2}.
    \label{eq:temporal_metrics}
\end{align}
Macro metrics average category values equally, with the square root taken within each category for RMSE.

\subsection{Temporal Controls without External Scores}
\label{sec:temporal_controls}

All controls use the identical fixed model--task panels, early checkpoints, score-space squared-error loss, and later evaluation points. The persistence control retains the last early score. Independent time curves fit $\sigma(A_m+B_mg(\tau))$ separately to each model's early category means, using only the numerical intercept and slope penalties in equation~\ref{eq:temporal_objective}. The model-indicator control replaces the five PC coordinates with one-hot model indicators and uses the same early-only regularization and shared-versus-model-dependent slope selection as the capability response. Its final coefficients are refitted on all early points. Table~\ref{tab:temporal_controls} retains every category.

\begin{table}[htbp]
\centering
\caption{Later-score RMSE in points under matched temporal evaluation. Indicators denotes regularized model identities; PCs denotes the selected five-PC response.}
\label{tab:temporal_controls}
\begin{tabular}{lrrrr}
\toprule
Category & Last score & Independent & Indicators & PCs\\
\midrule
CUDA & 4.19 & 4.95 & 4.18 & 3.75\\
Model Development & 6.19 & 3.46 & 6.06 & 3.79\\
Puzzle \& Challenge & 5.83 & 4.27 & 4.41 & 4.08\\
System Optimization & 1.71 & 3.97 & 3.41 & 3.46\\
Formal & 4.40 & 1.86 & 1.66 & 1.93\\
Games & 1.95 & 1.10 & 1.01 & 1.03\\
Knowledge & 2.17 & 2.05 & 1.90 & 1.53\\
Optimization & 1.11 & 0.58 & 0.50 & 0.57\\
Scientific \& ML & 2.25 & 1.04 & 1.50 & 1.17\\
Systems \& SE & 1.19 & 1.90 & 1.91 & 1.86\\
\midrule
Macro average & 3.10 & 2.52 & 2.66 & 2.31\\
\bottomrule
\end{tabular}
\end{table}

With four or five models per category, the intercept-plus-PC design has full row rank; regularization selects among coefficient vectors with identical fitted indices. The coefficients describe regularized associations, with predictive performance evaluated by the later-score controls in Table~\ref{tab:temporal_controls}.

\section{Gain Decomposition and Opportunity Modeling}
\label{sec:coverage_details}

\subsection{The Coverage Identity and Its Change over Time}

For positive gains, $M_c C$ is the order-two Hill effective number~\citep{hill1973diversity}, $(\sum_m\Delta_m)^2/\sum_m\Delta_m^2$; dividing by $M_c$ gives normalized gain coverage. Fix a task and interval, and suppress their indices. Let $K_+$ be the number of positive increments, $\bar\Delta_+$ their mean, and $v_+$ their variance with divisor $K_+$. For $K_+>0$,
\begin{equation}
    C=\frac{(K_+\bar\Delta_+)^2}{M_cK_+(\bar\Delta_+^2+v_+)}
     =\frac{K_+}{M_c}\frac{1}{1+v_+/\bar\Delta_+^2}=P\,U.
\end{equation}
When $K_+=0$, we set $P=C=0$ and leave the positive-gain balance undefined. For a task with positive gains in both its early and late intervals, denoted by subscripts $\mathrm{early}$ and $\mathrm{late}$, the exact symmetric decomposition is
\begin{align}
    C_{\mathrm{late}}-C_{\mathrm{early}}
    ={}&(P_{\mathrm{late}}-P_{\mathrm{early}})
                  \frac{U_{\mathrm{late}}+U_{\mathrm{early}}}{2}\nonumber\\
       &+(U_{\mathrm{late}}-U_{\mathrm{early}})
                  \frac{P_{\mathrm{late}}+P_{\mathrm{early}}}{2}.
    \label{eq:coverage_change}
\end{align}
The two terms exactly decompose the observed coverage change into participation and balance components. We compute them per task, then average within category and across categories. The early and late intervals are $20$--$40\%$ and $80$--$100\%$ for AutoLab, and $2$--$4$ h and $10$--$12$ h for EdgeBench. Participation uses all $71$ fixed tasks. The paired decomposition uses the $62$ tasks with positive gains in both intervals, giving contributions of $-23.95$ percentage points from participation and $+2.92$ from balance. A nonparametric bootstrap~\citep{efron1979bootstrap} resamples whole tasks within categories, keeping the model set fixed, and gives a $95\%$ interval of $29.5$--$40.8$ percentage points for the macro participation decline. Raising the improvement threshold from numerical zero to $0.01$ or $0.1$ score points preserves the decline in all ten category means.

\paragraph{Expanded trajectory panels.}
Table~\ref{tab:participation_panel_sensitivity} extends the analysis from the primary panels to all complete task--model curves available at the panel-selection stage. For each task $i$, participation divides the number of improving models by the number observed on that task; this model set stays fixed across the first and last intervals. We average tasks within categories, then the ten categories equally. The expansion increases coverage from $71$ tasks and $329$ task--model pairs to $80$ tasks and $402$ pairs, without imputing incomplete curves. Participation declines in every category under both definitions, averaging $35.2$ and $33.5$ percentage points, respectively.

\begin{table}[!htbp]
\centering
\caption{Participation declines under both panel definitions. Counts are tasks; decline is the task-averaged first-to-last reduction in percentage points. Expanded panels retain every available complete task--model curve, with each task\textquotesingle s model set fixed across intervals.}
\label{tab:participation_panel_sensitivity}
\small
\begin{tabular*}{\linewidth}{@{\extracolsep{\fill}}lrrrr@{}}
\toprule
 & \multicolumn{2}{c}{Fixed primary panel} & \multicolumn{2}{c}{Expanded panel}\\
\cmidrule(lr){2-3}\cmidrule(lr){4-5}
Category & Tasks & Decline (pp) & Tasks & Decline (pp)\\
\midrule
CUDA & 3 & $41.7$ & 4 & $32.5$\\
Model Development & 2 & $50.0$ & 2 & $42.5$\\
Puzzle \& Challenge & 5 & $30.0$ & 9 & $43.0$\\
System Optimization & 12 & $54.2$ & 14 & $48.3$\\
\midrule
Formal & 8 & $20.0$ & 8 & $20.0$\\
Games & 8 & $27.5$ & 8 & $27.5$\\
Knowledge & 4 & $40.0$ & 4 & $40.0$\\
Optimization & 14 & $21.4$ & 15 & $20.0$\\
Scientific \& ML & 4 & $25.0$ & 4 & $21.2$\\
Systems \& SE & 11 & $41.8$ & 12 & $40.4$\\
\midrule
Ten-category mean & & $35.2$ & & $33.5$\\
\bottomrule
\end{tabular*}
\end{table}

\subsection{Opportunity Response and Later Calibration}
\label{sec:opportunity_calibration}

For each model--category--interval observation, let $f_{mc,k}$ be its empirical fraction of tasks with positive increments. The opportunity model is logistic regression on log predicted mean gain:
\begin{equation}
    \operatorname{logit}\widehat\pi_{mc,k}
       =\gamma\bigl(\log\widehat\mu_{mc,k}-\log\mu_c^\star\bigr).
    \label{eq:opportunity_logit}
\end{equation}
In implementation, log gain is standardized on early observations; the reported scale and exponent absorb that standardization. A category intercept and one shared slope minimize binomial cross-entropy against $f_{mc,k}$, with equal total category weight and weights proportional to task counts within category. Only the standardized slope receives an L2 penalty, selected from $\{0,0.001,0.01,0.1\}$ using the last early interval and category-macro Brier score~\citep{brier1950verification}. The slope is fitted without a sign constraint and is positive in the selected fit. Refitting the opportunity parameters under whole-task resampling gives a $95\%$ interval of $1.54$--$2.05$ for $\gamma$, conditional on the fixed capability curves and selected regularization.

Figure~\ref{fig:gain_opportunity_relation}a uses the $20,40,60,80$th percentiles of early log gain ratios as five-bin boundaries, extending the outer bins to include all later points. Each square places the geometric mean gain ratio against the unweighted mean observed task fraction within that bin. From left to right, the groups contain $60,18,5,6,1$ later observations; all $90$ enter these descriptive summaries. The orange segments connect adjacent group averages in increasing predicted-gain order; they do not represent a trajectory through time. The black response curve is unchanged.

Figure~\ref{fig:growth_opportunity_profiles}(b) evaluates later calibration with six bins defined by the early predicted-probability quantiles at $j/6$, $j=1,\ldots,5$, extending the outer boundaries to include all later predictions. Five bins contain later observations. Squares show task-count-weighted mean predicted probabilities and pooled observed improvement fractions. Figure~\ref{fig:gain_opportunity_relation}b adds all $90$ model--category--interval fractions as individual points. The groups contain $51,22,9,5,3$ fractions, representing $372,168,64,30,24$ task--model--interval entries, respectively. Error bars are $95\%$ percentile intervals from $5{,}000$ whole-task bootstrap replicates within categories, retaining each sampled task's models and intervals; predictions and bin boundaries remain fixed.

\paragraph{Matched input comparison.}
To evaluate the input to equation~\ref{eq:opportunity_response}, we replace log predicted gain with log interval midpoint, retaining category intercepts, a shared slope, the early fitting intervals, and the regularization-selection procedure. Both responses use these category intercepts and a single input, without model-specific intercepts. On all $90$ later model--category--interval observations, category-macro Brier scores are $0.2071$ for predicted gain and $0.2187$ for time. For an observed improving-task fraction $f$ and prediction $\widehat\pi$, the Brier contribution is $f(1-\widehat\pi)^2+(1-f)\widehat\pi^2$; this averages squared probability error over the tasks' binary improvement outcomes. We average contributions within each category and then across the ten categories.

\begin{figure}[tbp]
\centering
\includegraphics[width=\linewidth]{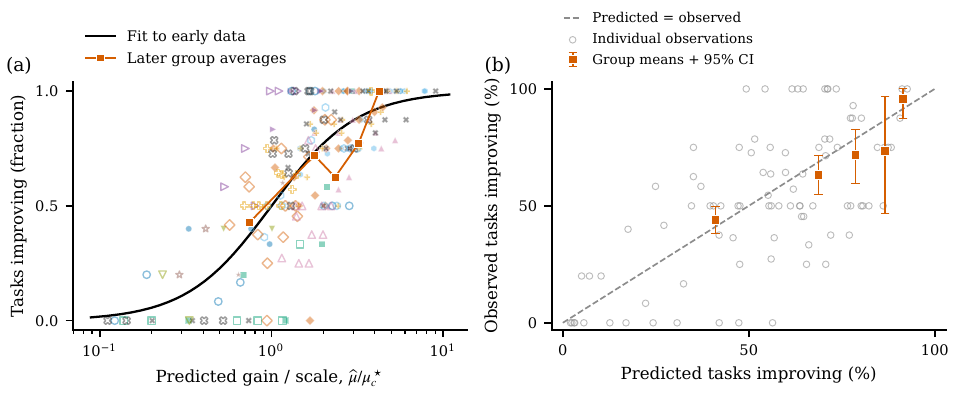}
\caption{The gain--opportunity response and later calibration. (a) All $209$ model--category--interval observations (filled: early; open: later). The black curve fits early improvement frequencies; orange squares connect later averages in five groups defined by early gain-ratio quantiles. (b) All $90$ later fractions (gray), with the same task-weighted means and $95\%$ whole-task bootstrap intervals as Figure~\ref{fig:growth_opportunity_profiles}(b); dashed line: predicted $=$ observed.}
\label{fig:gain_opportunity_relation}
\end{figure}

\subsection{Task-Grounded Interpretation of Progress}
\label{sec:task_signature_details}

The task-signature analysis relates the same $29$ external measurements to observed trajectories on the fixed $71$-task panel. Table~\ref{tab:task_content} summarizes the work performed in each category. All external measurements named in the main-text cases have published scores for every model in their respective panel. Spearman $\rho_s$ is computed across the four or five models within a category, with ties assigned average ranks. These exploratory associations describe task- and stage-related performance patterns; the task counts do not constitute additional independent model samples.

\begin{table}[htbp]
\centering
\small
\setlength{\tabcolsep}{3pt}
\caption{Work performed by the fixed task panels. Counts are models/tasks. The first four categories are AutoLab; the remaining six are EdgeBench. The machine-readable supplement retains all 71 task descriptions.}
\label{tab:task_content}
\begin{tabular}{llp{7.0cm}}
\toprule
Category & Models/tasks & Task process\\
\midrule
CUDA & 4/3 & Optimize decoding, geometric correspondence, and elliptic-curve CUDA kernels.\\
Model Development & 4/2 & Select fine-tuning data and optimize an online serving engine.\\
Puzzle \& Challenge & 4/5 & Reduce program or model complexity under correctness and accuracy constraints.\\
System Optimization & 4/12 & Improve implementations of cryptography, retrieval, storage, and numerical kernels.\\
Formal & 5/8 & Complete interdependent Lean/Coq proofs in multi-file projects.\\
Games & 5/8 & Explore interactive fiction or implement game-playing and management agents.\\
Knowledge & 5/4 & Build three domain-specific systems and produce one examination-document suite.\\
Optimization & 5/14 & Implement and revise solvers for routing, packing, constraints, and search.\\
Scientific \& ML & 4/4 & Implement and evaluate reinforcement learning, inverse models, and graph learning.\\
Systems \& SE & 5/11 & Implement repositories and specified features; optimize throughput and kernels.\\
\bottomrule
\end{tabular}
\end{table}

For a terminal window containing $w$ intervals and an improvement threshold $\epsilon$ in score points, define
\begin{equation}
    f^{(w,\epsilon)}_{mc}
    =\frac{1}{w|\mathcal I_c|}
      \sum_{k=N-w}^{N-1}\sum_{i\in\mathcal I_c}
      \mathbf 1\{100\Delta_{mi,k}>\epsilon\}.
    \label{eq:terminal_improvement_fraction}
\end{equation}
Thus each interval contributes its fraction of tasks improving, and the selected intervals receive equal weight. The main setting uses $w=2$ and positive gain. Sensitivity checks cross $w\in\{1,2,3\}$ with $\epsilon\in\{0,0.1,0.5\}$, retaining every task and model. The three suffixes cover $80$--$100\%$, $60$--$100\%$, and $40$--$100\%$ for AutoLab, and $10$--$12$, $8$--$12$, and $6$--$12$ hours for EdgeBench. They overlap; the three-interval suffix can contain fitted checkpoints. This is a descriptive window check, with the temporal fitting protocol unchanged.

\paragraph{Aggregate profiles across all external measurements.}
Table~\ref{tab:aggregate_capability_profiles} averages the signed Spearman correlations over all $29$ benchmarks within each category, then across categories with equal weight. Final scores have a mean association of $0.454$, compared with $0.186$ for terminal improvement frequency, a difference of $0.268$. The difference is positive in seven categories; Games, Scientific \& ML, and Systems \& SE have more positive terminal-frequency associations. The aggregate contrast motivates separate profiles for attainment and continued improvement, with the constituent benchmarks varying by task category.

The difference remains positive in four descriptive sensitivity checks. Removing each model from every category in which it appears and recomputing the correlations gives differences of $0.089$--$0.355$. The nine window--threshold combinations above give $0.268$--$0.432$. Giving equal weight to distinct benchmark rankings within each category yields $0.202$. Retaining only benchmarks with published scores for every model in the category yields $0.279$, using $9$--$27$ measurements per category. These checks preserve equal category weighting; undefined correlations from constant rankings are omitted within a category. They measure the stability of the observed aggregate contrast across model membership and measurement choices.

\begin{table}[!htbp]
\centering
\caption{Attainment and terminal improvement have different external benchmark profiles. Each entry averages signed model-rank correlations over all $29$ measurements within a category; the final row averages categories equally. Difference is final-score minus terminal-frequency association.}
\label{tab:aggregate_capability_profiles}
\small
\begin{tabular*}{\linewidth}{@{\extracolsep{\fill}}lrrr@{}}
\toprule
Category & Final score & Terminal frequency & Difference\\
\midrule
CUDA & $0.290$ & $-0.222$ & $+0.511$\\
Model Development & $0.250$ & $-0.250$ & $+0.500$\\
Puzzle \& Challenge & $0.473$ & $-0.364$ & $+0.837$\\
System Optimization & $0.304$ & $-0.051$ & $+0.354$\\
\midrule
Formal & $0.671$ & $0.171$ & $+0.499$\\
Games & $0.444$ & $0.601$ & $-0.156$\\
Knowledge & $0.538$ & $0.259$ & $+0.279$\\
Optimization & $0.616$ & $0.392$ & $+0.224$\\
Scientific \& ML & $0.339$ & $0.657$ & $-0.319$\\
Systems \& SE & $0.616$ & $0.671$ & $-0.055$\\
\midrule
Ten-category mean & $0.454$ & $0.186$ & $+0.268$\\
\bottomrule
\end{tabular*}
\end{table}

\paragraph{Composite profiles and their constituent measurements.}
Tables~\ref{tab:composite_profiles_main} and~\ref{tab:composite_profiles_additional} list every benchmark with $|\rho_s|\geq0.70$ under the stated outcome, grouping equal rounded correlations for readability. This cutoff indexes the descriptive scan; it neither selects predictor inputs nor limits the number of profile members. All $290$ category--benchmark pairs, their nine-condition results, published-score masks, and within-panel benchmark rank equivalences are retained in the numerical supplement. For example, Optimization's NL2Repo, SWE-bench Pro, and Claw Eval share one ranking, while APEX Agents and CyberGym share another. Their different task content broadens the interpretation of the profile; the benchmark count is not an independent-sample count.

\begin{table}[htbp]
\centering\small
\setlength{\tabcolsep}{3pt}
\caption{Composite profiles in the main-text categories. All external measurements with $|\rho_s|\geq0.70$ for terminal improvement frequency are listed, with no cap on profile size. Parentheses give Spearman correlations; category counts are models/tasks. The default uses the last two intervals and positive gain; the additional increment threshold is stated explicitly. An asterisk identifies a benchmark with at least one frozen estimated input in that panel. Correlated members describe a joint performance profile, not separate effects.}
\label{tab:composite_profiles_main}
\begin{tabular}{p{2.05cm}p{9.0cm}}
\toprule
\raggedright Category & \raggedright Concrete benchmark members and observed associations\tabularnewline
\midrule
\raggedright Systems \& SE (5/11) & \raggedright BrowseComp, FrontierScience-Olympiad*, HLE, IMO-AnswerBench, SimpleQA, SWE-bench Multilingual, Terminal-Bench, Toolathlon (+0.70); AIME 2026, DeepSWE, SciCode (+0.80); HMMT Feb 2026, HMMT Nov 2025 (+0.82); MathArena, $\tau^3$-Banking (+0.90); CritPt, FrontierScience-Research, GPQA (+1.00)\tabularnewline[4pt]
\raggedright Optimization (5/14) & \raggedright DeepSWE, DeepSearchQA (+0.72); APEX Agents, CyberGym, MCPAtlas (+0.82); Claw Eval, NL2Repo, SWE-bench Pro (+0.97)\tabularnewline[4pt]
\raggedright Knowledge (5/4) & \raggedright AA-LCR (-0.74); APEX Agents, CyberGym, MCPAtlas (+0.74); Claw Eval, NL2Repo, SWE-bench Pro (+0.95)\tabularnewline[4pt]
\raggedright System Optimization (4/12) & \raggedright MathArena, MMLU-Pro (-1.00); GPQA, SciCode, SimpleQA (-0.80); AA-LCR, HMMT Nov 2025* (-0.74); IMO-AnswerBench (+1.00)\tabularnewline[4pt]
\bottomrule
\end{tabular}
\end{table}

Optimization's eight-member profile spans repository construction, agent execution, and information seeking. Seven members stay positive in all nine conditions; DeepSWE is positive in seven. In Systems \& SE, final-score correlations are $0.90$ for NL2Repo, SWE-bench Pro, Claw Eval, DeepSWE, Terminal-Bench, and SWE-bench Multilingual. The first three have terminal-frequency correlations of $0.30$, while DeepSWE remains at $0.80$. Scientific reasoning (GPQA, CritPt, FrontierScience Research), mathematics (MathArena, AIME), and scientific programming (SciCode) are all positive across the nine conditions. This is an overlapping capability profile, with distinct associations for attainment and continued improvement.

Formal exhibits increment-size dependence within the same terminal window. AA-LCR correlations are $0.80$, $0.87$, and $0.67$ at thresholds $0$, $0.1$, and $0.5$ points. At $>0.5$ points, AIME, SciCode, and $\tau^3$-Banking each correlate at $0.97$; HMMT February at $0.95$; GPQA, CritPt, and FrontierScience Research at $0.87$. Thus long-input integration, mathematical/scientific reasoning, programming, and multi-step execution provide complementary task interpretations. The thresholds distinguish observed increment sizes, not time stages or proof difficulty.

Table~\ref{tab:capability_window_sensitivity} provides individual-member window checks. These retain the original measurements and relate task processes to their observed progress patterns.

\begin{table}[htbp]
\centering
\small
\setlength{\tabcolsep}{3pt}
\caption{Window and gain-threshold sensitivity of the representative associations. Entries are Spearman correlations across fixed model panels with the fraction of tasks improving, averaged over the indicated final intervals. Last two is the main window; thresholds are in 0--100 score points. The first five rows concern AutoLab, the remainder EdgeBench. Every displayed external measurement is published for all models in its panel.}
\label{tab:capability_window_sensitivity}
\begin{tabular}{llrrrrrrrrr}
\toprule
Category & Benchmark & \multicolumn{3}{c}{$100\Delta>0$} & \multicolumn{3}{c}{$100\Delta>0.1$} & \multicolumn{3}{c}{$100\Delta>0.5$}\\
 & & Last 1 & Last 2 & Last 3 & Last 1 & Last 2 & Last 3 & Last 1 & Last 2 & Last 3\\
\midrule
CUDA & SWE Pro & $+0.89$ & $+0.74$ & $+0.74$ & $+0.74$ & $+0.60$ & $+0.60$ & $+0.74$ & $+0.60$ & $+0.60$\\
Model Dev. & MathArena & $-0.45$ & $-0.80$ & $-0.80$ & $-0.45$ & $-0.80$ & $-0.80$ & $-0.45$ & $-0.63$ & $-0.40$\\
Puzzle & SciCode & $-0.89$ & $-0.95$ & $-0.95$ & $-0.89$ & $-0.95$ & $-0.95$ & $-0.89$ & $-0.95$ & $-0.77$\\
System Opt. & MMLU-Pro & $-1.00$ & $-1.00$ & $-1.00$ & $-1.00$ & $-1.00$ & $-1.00$ & $-0.95$ & $-1.00$ & $-1.00$\\
System Opt. & AA-LCR & $-0.74$ & $-0.74$ & $-0.74$ & $-0.74$ & $-0.74$ & $-0.74$ & $-0.89$ & $-0.74$ & $-0.74$\\
\midrule
Formal & AA-LCR & $+0.67$ & $+0.80$ & $+0.87$ & $+0.67$ & $+0.87$ & $+0.82$ & $+0.67$ & $+0.67$ & $+0.60$\\
Games & GPQA & $+0.67$ & $+0.97$ & $+0.97$ & $+0.89$ & $+0.97$ & $+1.00$ & $+0.90$ & $+0.90$ & $+0.97$\\
Knowledge & NL2Repo & $+0.89$ & $+0.95$ & $+0.79$ & $+0.53$ & $+0.67$ & $+0.74$ & $+0.53$ & $+0.36$ & $+0.82$\\
Optimization & NL2Repo & $+0.89$ & $+0.97$ & $+0.90$ & $+0.67$ & $+0.40$ & $+0.82$ & $+0.58$ & $+0.50$ & $+0.80$\\
Scientific/ML & DeepSWE & $+0.77$ & $+1.00$ & $+1.00$ & $+0.00$ & $+0.63$ & $+0.80$ & $-0.77$ & $+0.89$ & $+0.74$\\
Systems/SE & DeepSWE & $+0.67$ & $+0.80$ & $+0.82$ & $+0.67$ & $+0.67$ & $+0.56$ & $+0.70$ & $+0.36$ & $+0.31$\\
Systems/SE & SWE Pro & $+0.10$ & $+0.30$ & $+0.41$ & $+0.10$ & $+0.10$ & $+0.10$ & $+0.30$ & $-0.21$ & $-0.21$\\
\bottomrule
\end{tabular}
\end{table}

\paragraph{Shared processes and additional task profiles.}
Knowledge's six positive members combine repository engineering (NL2Repo, SWE-bench Pro) with execution reliability and tool-mediated work (Claw Eval, APEX Agents, CyberGym, MCPAtlas); all remain positive across the nine conditions. This profile matches three system-building tasks among four deliverables. Games combines reasoning, search, and interaction across thirteen published positive members, all positive in the nine conditions; its three interactive-fiction tasks and five programmed agents share exploration while differing in implementation demands. Scientific \& ML has seventeen published positive members spanning mathematics, science, implementation, and execution. Its profile depends on increment size and window: DeepSWE reverses to $-0.77$ for final-interval increments above $0.5$ points. CUDA and Model Development retain their engineering--tool profiles and opposing mathematical or knowledge associations in Table~\ref{tab:composite_profiles_additional}; their three- and two-task panels also admit task-level interpretation. Puzzle \& Challenge's published negative members, SciCode, GPQA, and CritPt, all correlate positively with initial performance ($0.80$) and negatively with terminal frequency ($-0.95$).

\begin{table}[htbp]
\centering\small
\setlength{\tabcolsep}{3pt}
\caption{Composite profiles in the remaining six categories. All external measurements with $|\rho_s|\geq0.70$ for terminal improvement frequency are listed, with no cap on profile size. Parentheses give Spearman correlations; category counts are models/tasks. The default uses the last two intervals and positive gain; the additional increment threshold is stated explicitly. An asterisk identifies a benchmark with at least one frozen estimated input in that panel. Correlated members describe a joint performance profile, not separate effects.}
\label{tab:composite_profiles_additional}
\begin{tabular}{p{2.05cm}p{9.0cm}}
\toprule
\raggedright Category & \raggedright Concrete benchmark members and observed associations\tabularnewline
\midrule
\raggedright Formal (5/8) & \raggedright $\tau^3$-Banking (+0.70); AA-LCR (+0.80)\newline $>0.5$ points: CritPt, FrontierScience-Research, FrontierScience-Olympiad*, GPQA (+0.87); HMMT Feb 2026 (+0.95); AIME 2026, SciCode, $\tau^3$-Banking (+0.97)\tabularnewline[4pt]
\raggedright Games (5/8) & \raggedright HMMT Nov 2025 (+0.71); AA-LCR, AIME 2026, SciCode (+0.72); HMMT Feb 2026 (+0.76); BrowseComp, IMO-AnswerBench, MathArena, SimpleQA (+0.82); $\tau^3$-Banking (+0.87); CritPt, FrontierScience-Research, GPQA (+0.97)\tabularnewline[4pt]
\raggedright Scientific \& ML (4/4) & \raggedright APEX Agents, BrowseComp, CritPt, CyberGym, DeepSearchQA, FrontierScience-Research, GPQA, HLE, MMLU-Pro, SimpleQA, SWE-bench Multilingual, SWE-bench Verified, Terminal-Bench, Toolathlon (+0.80); HMMT Nov 2025 (+0.95); DeepSWE, MathArena (+1.00)\tabularnewline[4pt]
\raggedright CUDA (4/3) & \raggedright HMMT Feb 2026, MathArena (-0.95); BrowseComp*, FrontierScience-Olympiad, IFEval*, SimpleQA (-0.74); MCPAtlas, SWE-bench Multilingual, SWE-bench Pro (+0.74)\tabularnewline[4pt]
\raggedright Model Development (4/2) & \raggedright FrontierScience-Research, HMMT Feb 2026, SimpleQA (-1.00); BrowseComp*, MathArena (-0.80); SWE-bench Multilingual (+0.80)\newline $>0.5$ points: BrowseComp*, FrontierScience-Research, HMMT Feb 2026, SimpleQA (-0.95); MCPAtlas, SWE-bench Multilingual, SWE-bench Pro (+0.74)\tabularnewline[4pt]
\raggedright Puzzle \& Challenge (4/5) & \raggedright AIME 2026*, CritPt, FrontierScience-Olympiad*, GPQA, HMMT Feb 2026*, MathArena*, MMLU-Pro*, SciCode, SimpleQA* (-0.95)\tabularnewline[4pt]
\bottomrule
\end{tabular}
\end{table}

\begin{figure}[htbp]
\centering
\includegraphics[width=\linewidth]{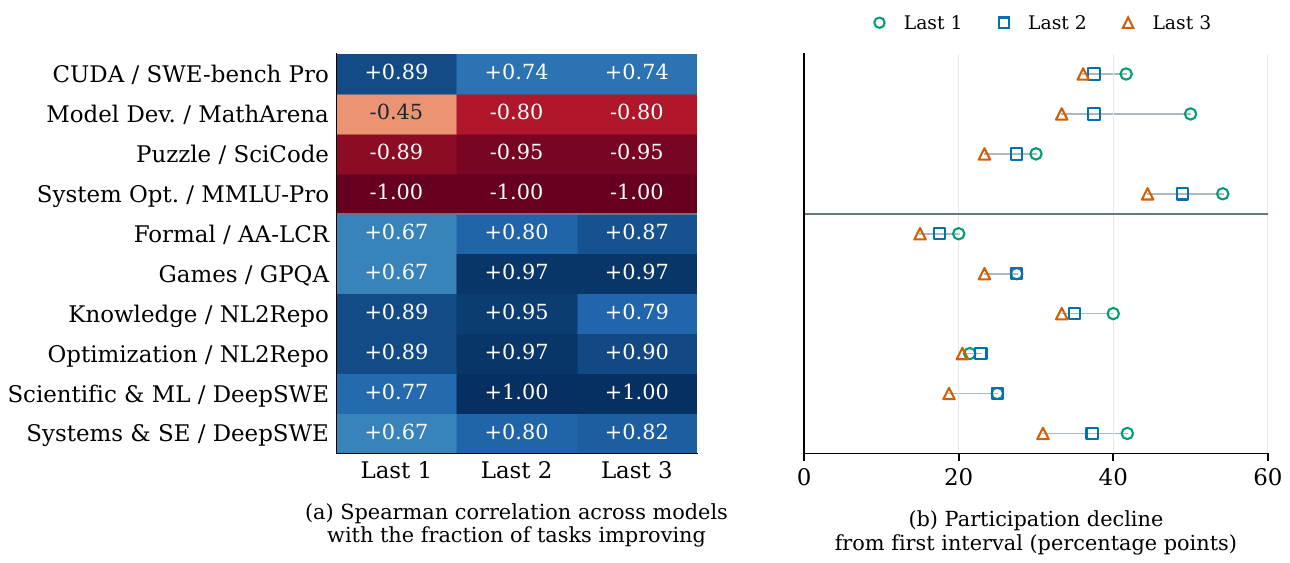}
\caption{Window sensitivity on all ten categories. (a) Benchmark correlations with the equally weighted mean fraction of tasks improving across the selected final intervals, using positive gain. The displayed pairs come from the task-interpretation analysis. (b) Task-averaged model participation decline relative to the first interval, in percentage points. The horizontal rule separates AutoLab from EdgeBench; model/task counts follow Table~\ref{tab:task_content}. Last two is the main window. All tasks and models are retained, and the response models are unchanged.}
\label{fig:capability_window_sensitivity}
\end{figure}

Across all three windows, category-mean participation and effective gain coverage remain below their first-interval values in every category. For the last one, two, and three intervals, macro participation declines are $35.2$, $31.7$, and $27.9$ percentage points; coverage declines are $23.6$, $20.5$, and $18.2$ points. Coverage is computed per task and interval before averaging, retaining all-zero cases as zero.

\subsection{Gain Realization over the Observed Trajectory}
\label{sec:gain_timing}

To describe gain timing without an early/late cutoff, normalize the observed window by $u_k=(\tau_k-\tau_0)/(1-\tau_0)$ and let $\overline\Delta_{mc,k}=\bar s_{mc}(\tau_{k+1})-\bar s_{mc}(\tau_k)$. The gain-weighted interval midpoint is
\begin{equation}
    \ell^{\mathrm{gain}}_{mc}
    =\frac{\sum_{k=0}^{N-1}\overline\Delta_{mc,k}(u_k+u_{k+1})/2}
           {\sum_{k=0}^{N-1}\overline\Delta_{mc,k}}.
    \label{eq:gain_weighted_time}
\end{equation}
Smaller values indicate earlier realization of the model's observed gains; constant-rate improvement gives $0.5$. The statistic locates gains at recorded interval midpoints rather than asserting within-interval event times. It is defined for $44$ of the $45$ model--category pairs. GPT-5.5 has zero total gain in Model Development, leaving three effective models for that category's timing correlations. Supplementary tables report effective counts and attainment correlations on the same effective model set.

\begin{table}[htbp]
\centering
\small
\setlength{\tabcolsep}{3pt}
\caption{Gain realization on all four models in AutoLab System Optimization (12 tasks). Scores use the 0--100 scale. The fourth column is the fraction of each model\textquotesingle s observed 20--100\% gain realized by 60\% budget. The timing index $\ell^{\mathrm{gain}}$ weights normalized interval midpoints by category mean-score increments; smaller values indicate earlier gains.}
\label{tab:system_gain_timing}
\begin{tabular}{lrrrr}
\toprule
Model & First score & Final score & Gain by 60\% & $\ell^{\mathrm{gain}}$\\
\midrule
Opus 4.7 & 63.96 & 68.55 & 94.8\% & 0.179\\
DeepSeek V4 Pro & 31.40 & 50.92 & 91.0\% & 0.255\\
Gemini 3.1 Pro & 51.13 & 60.45 & 100.0\% & 0.125\\
GLM 5.2 & 50.63 & 63.52 & 71.5\% & 0.313\\
\bottomrule
\end{tabular}
\end{table}

Table~\ref{tab:system_gain_timing} reports all four models in System Optimization. The same six published benchmarks associate negatively with both terminal improvement frequency and $\ell^{\mathrm{gain}}$: AA-LCR at $-0.74$, GPQA/SciCode/SimpleQA at $-0.80$, and MMLU-Pro/MathArena at $-1.00$. All six retain negative terminal-frequency associations in every window--threshold condition. Their first-score correlations differ: AA-LCR is $0.95$, GPQA/SciCode $0.80$, MMLU-Pro/MathArena $0.60$, and SimpleQA $0.00$. The shared profile thus concerns gain timing, without requiring identical initial-attainment relationships. The fraction of observed $20$--$100\%$ gain realized by $60\%$ budget ranges from $71.5\%$ to $100\%$. IMO-AnswerBench retains the opposite timing association ($1.00$), distinguishing this profile from a uniform mathematics effect. Puzzle \& Challenge provides a second example: SciCode correlates positively with initial performance ($0.80$) and negatively with $\ell^{\mathrm{gain}}$ ($-0.80$). Timing uses complete observed trajectories for retrospective interpretation; continuation conditions on the currently visible prefix.

\subsection{Joint Readouts and Observed Benchmark Profiles}

The joint predictor combines all $29$ inputs through the five PCs. Expanding its readouts gives $\boldsymbol\beta_c^A=\mathbf V_5\mathbf w_c$ and $\boldsymbol\beta_c^B=\mathbf V_5\mathbf v_c$ in equation~\ref{eq:benchmark_readout}; the numerical supplement supplies both full vectors. Direct PC evaluation and this expansion agree within $9\times10^{-16}$ in the fitted indices. The joint readouts predict growth, while empirical benchmark profiles characterize task demands and gain timing through associations with observed outcomes. These profiles are descriptive associations rather than independent benchmark effects; larger-increment profiles in Table~\ref{tab:composite_profiles_additional} concern a different outcome from the opportunity response's any-positive improvement.

\section{Continuation Fitting and the Sequential Value Bound}
\label{sec:continuation_fitting}

\subsection{The Monotone Decision Reference}

The decision reference uses the same external coordinates and log-time functional form as the category response, with an analysis-specific first checkpoint. It is refitted on source-model histories with a positive slope:
\begin{align}
    A^{\mathrm{ref}}_{mc}&=a^{\mathrm{ref}}_c+
            (\mathbf w^{\mathrm{ref}}_c)^{\mathsf T}\mathbf z_m,
    &B^{\mathrm{ref}}_{mc}&=\operatorname{softplus}\!\left(
            b^{\mathrm{ref}}_c+(\mathbf v^{\mathrm{ref}}_c)^{\mathsf T}\mathbf z_m\right),
    \label{eq:reference_readouts}\\
    r_{mc}(\tau)&=\sigma\!\left(A^{\mathrm{ref}}_{mc}+B^{\mathrm{ref}}_{mc}g(\tau)\right).
    \label{eq:reference_response}
\end{align}
Here $\operatorname{softplus}(u)=\log(1+e^u)$ guarantees $B^{\mathrm{ref}}_{mc}>0$. AutoLab uses $\tau_0=0.05$ for the first declared decision checkpoint and $\tau_0=0.2$ for category forecasting; EdgeBench uses $\tau_0=2/12$ for both. These decision grids and source-history fits are distinct from the early-checkpoint category forecasts in Table~\ref{tab:temporal_protocol}. The fitted category response $\widehat q$ and the decision reference $r$ share their external representation and functional structure, with coefficients learned under their respective information conditions.

\subsection{Source-Only Estimation}
\label{sec:source_estimation}

We fit a separate continuation model within each benchmark suite. Source category means average replicates within a model--task--checkpoint and then average available tasks within category. The monotone reference in equation~\ref{eq:reference_response} minimizes score squared error with equal total model weight within category and coefficient penalty $0.01$ on PC readouts; intercepts are unpenalized. Reference fits use completed source histories, including declared endpoints. Missing earlier observations are not filled with future scores.

Let $\widetilde{\mathbf h}_k\in\mathbb R^7$ contain standardized budget progress, current and initial scores, last-interval and last-two-interval gains, observed-prefix improvement rate, and trailing-stagnation fraction. The positive response coefficients are
\begin{equation}
    \log\alpha_k=u_{\alpha,0}+\mathbf u_\alpha^{\mathsf T}\widetilde{\mathbf h}_k+d_i^\alpha,
    \qquad
    \log\beta_k=u_{\beta,0}+\mathbf u_\beta^{\mathsf T}\widetilde{\mathbf h}_k+d_i^\beta.
    \label{eq:conditional_parameters}
\end{equation}
Within each suite, the weights are shared across source models and tasks. Regularized task effects $d_i^\alpha,d_i^\beta$ are learned from other models' histories on task $i$. We jointly fit the next-checkpoint ($\mathrm n$) and endpoint ($\mathrm e$) gains in completed source runs. For $H_k>0$, their normalized targets are
\begin{equation}
    Z_{k,\mathrm n}=\frac{Y_{k+1}-Y_k}{H_k},
    \qquad Z_{k,\mathrm e}=\frac{Y_N-Y_k}{H_k},
    \label{eq:gain_targets}
\end{equation}
Both targets are zero when no score space remains. Equal-weight squared errors on these two fractions train the same curve, including zero-gain histories, with parameter regularization and source-only model selection.

Index active training states by $j$ and let $\boldsymbol\phi_j$ concatenate an intercept, the seven source-standardized state features, and task indicators multiplied by $\sqrt{n_{\mathrm{task}}}$, where $n_{\mathrm{task}}$ is the number of source tasks. Let $\boldsymbol\theta_\alpha,\boldsymbol\theta_\beta$ be the corresponding coefficient vectors. This design implements equation~\ref{eq:conditional_parameters}; its task effects satisfy $d_i^\alpha=\sqrt{n_{\mathrm{task}}}\,\theta_{\alpha,i}$ and likewise for the shape. Write $F_{j,h}=1-\exp[-\alpha_jD_j(\tau_{j,h})^{\beta_j}]$. The objective is
\begin{equation}
    \mathcal L_{\mathrm{gain}}
      =\frac12\sum_j\nu_j\sum_{h\in\{\mathrm n,\mathrm e\}}(F_{j,h}-Z_{j,h})^2
       +\rho\left(\|\boldsymbol\theta_{\alpha,\mathrm{pen}}\|_2^2
                         +\|\boldsymbol\theta_\beta\|_2^2\right).
    \label{eq:conditional_objective}
\end{equation}
The subscript $\mathrm{pen}$ denotes the intensity coefficients excluding the intercept; all shape coefficients are penalized. The fixed-shape case sets $\boldsymbol\theta_\beta=\mathbf0$. State weights $\nu_j$ are inverse trajectory lengths, rescaled so that each source model has equal total mass and their mean over states is one. The observed-prefix improvement-rate feature divides the number of positive observed increments by the number of observed intervals; the stagnation feature uses the trailing number of zero-gain intervals with the same denominator. Both are zero when no interval has yet been observed. Active full-score states remain in training with zero normalized targets. Numerically, remaining score space is floored at $10^{-10}$ for division and rescaling, and the decision uses a positive-value tolerance of $10^{-10}$; full-score predictions are therefore at most this tolerance and stop. Naturally ended states return exactly zero gain.

For each outer target-model holdout, we consider $\rho\in\{1,10,100\}$ and learned or fixed shape. Each candidate is evaluated by leaving out source models in turn. Its source-only selection score averages disagreement between the signs of $\max_h\widehat Q_{k,h}$ and the realized value $\max_h[H_kZ_{k,h}-\lambda c_{k,h}]$, weighted by the absolute realized value and averaged across prices. This is a realized-label selection criterion; the conditional-mean error in the theorem below is a different quantity. Ties prefer stronger regularization and then fixed shape. To construct each source training state's reference, we further exclude that state's model from reference fitting. The target reference is refitted using all outer source models. Source predictions used for price calibration repeat the model selection with each calibration model excluded.

\subsection{Replay Metrics and Operating Points}
\label{sec:operating_protocol}

\paragraph{Replay population and visible prefixes.}
AutoLab retains runs with a verified natural or budget termination, complete recorded evaluation outputs, a valid elapsed duration, and explicit initial-score feedback. A task--model pair enters replay when all three replicate runs meet these conditions, yielding $492$ runs. Decisions start at the first declared checkpoint at or after initial feedback becomes visible; each prefix score uses only evaluations recorded by that checkpoint. Runs without an available continuation checkpoint contribute terminal accounting only. Naturally terminated runs retain their final score and incur no further cost. EdgeBench retains the $252$ published task--model mean curves with all six two-hour checkpoints. Missing earlier observations are never filled with later scores. Replay retains eligible pairs without requiring a common task intersection across all models.

AutoLab continuation uses checkpoints at $5\%$, $10\%$, $15\%$, $25\%$, $50\%$, $75\%$, and $100\%$ of declared budget. EdgeBench uses two-hour checkpoints from $2$ to $12$ hours. The price grid contains zero and $61$ logarithmically spaced positive values from $10^{-5}$ to $30$ on the $[0,1]$ score scale, plus a run-to-end option.

For model $m$ and replicate $r$, cost $x$ is the sum of charged task hours divided by the sum of full-run hours for the same tasks, including the initial observation. Loss is the task-mean score shortfall in points. We retain nondominated parameter-sweep points and linearly interpolate the frontier $\ell^p_{mr}(x)$ for policy $p$. Let $[a_{mr},b_{mr}]$ be the reachable cost interval common to the compared policies within that model--replicate pair. The reported normalized frontier area is
\begin{equation}
    \mathcal A^p=\frac{1}{|\mathcal M|}\sum_{m\in\mathcal M}
       \frac{1}{R_m}\sum_{r=1}^{R_m}
       \frac{1}{b_{mr}-a_{mr}}\int_{a_{mr}}^{b_{mr}}\ell^p_{mr}(x)\,dx,
    \label{eq:normalized_frontier_area}
\end{equation}
Here $\mathcal M$ is the suite's evaluated model set, with $R_m=3$ for AutoLab and $R_m=1$ for EdgeBench's published mean trajectories. Dividing by interval width gives an average shortfall in score points; averaging then proceeds over replicates and models. Figure~\ref{fig:continuation_frontiers}a--b displays mean curves on the intersection of these intervals across models and replicates.

Table~\ref{tab:matched_tradeoffs} reads matched operating points from these mean curves. At time savings $s\in\{0.30,0.50\}$, retained score is $100[1-\bar\ell^p(1-s)/\bar S_{\mathrm{full}}]$, where $\bar\ell^p$ averages the interpolated replicate curves within models and then models. At retained fractions $q\in\{0.95,0.90\}$, we find the smallest cost $x$ satisfying $\bar\ell^p(x)\le(1-q)\bar S_{\mathrm{full}}$ and report $100(1-x)$. All crossings lie within the shared cost range; no extrapolation is used. These parameter-sweep comparisons summarize the frontiers, whereas Table~\ref{tab:frozen_operating} evaluates source-calibrated prices frozen for each target.

\begin{table}[htbp]
\centering
\caption{Mean score loss over matched cost ranges (score points; lower is better). Ranges are shared across policies for each model and replicate; results average replicates then models.}
\label{tab:continuation_results}
\small
\begin{tabular}{lrr}
\toprule
Policy & AutoLab & EdgeBench\\
\midrule
Fixed budget & $14.384$ & $2.442$\\
Time-based patience & $19.372$ & $3.233$\\
Recent gain (1 interval) & $22.907$ & $1.946$\\
Recent gain (2 intervals) & $24.271$ & $1.982$\\
Conditioned continuation & $\mathbf{8.825}$ & $\mathbf{1.774}$\\
\bottomrule
\end{tabular}
\end{table}

\begin{table}[!htb]
\centering
\caption{Matched tradeoffs on the mean frontiers in Figure~\ref{fig:continuation_frontiers}. For each suite: full-run score retained at $30\%/50\%$ time savings, and time saved at $95\%/90\%$ score retention. All entries are percentages; higher is better, best at displayed precision in bold. Values use interpolation within shared cost ranges.}
\label{tab:matched_tradeoffs_full}
\small
\begin{tabular*}{\linewidth}{@{\extracolsep{\fill}}lrrrrrrrr@{}}
\toprule
 & \multicolumn{4}{c}{\textbf{AutoLab}} & \multicolumn{4}{c}{\textbf{EdgeBench}}\\
\cmidrule(lr){2-5}\cmidrule(lr){6-9}
 & \multicolumn{2}{c}{Score retained} & \multicolumn{2}{c}{Time saved} & \multicolumn{2}{c}{Score retained} & \multicolumn{2}{c}{Time saved}\\
Method & $30\%$ & $50\%$ & $95\%$ & $90\%$ & $30\%$ & $50\%$ & $95\%$ & $90\%$\\
\midrule
Fixed budget & $89.9$ & $77.1$ & $17.2$ & $29.7$ & $94.7$ & $89.5$ & $28.5$ & $48.2$\\
Time-based patience & $81.3$ & $65.9$ & $8.2$ & $18.0$ & $92.8$ & $84.9$ & $24.6$ & $37.2$\\
Recent gain (1 interval) & $75.1$ & $58.6$ & $6.0$ & $12.1$ & $96.2$ & $91.6$ & $38.1$ & $54.9$\\
Recent gain (2 intervals) & $72.6$ & $54.7$ & $5.5$ & $10.9$ & $96.2$ & $\mathbf{91.8}$ & $34.9$ & $54.6$\\
Conditioned continuation & $\mathbf{97.3}$ & $\mathbf{90.5}$ & $\mathbf{38.1}$ & $\mathbf{50.9}$ & $\mathbf{96.8}$ & $\mathbf{91.8}$ & $\mathbf{38.4}$ & $\mathbf{55.9}$\\
\bottomrule
\end{tabular*}
\end{table}

For operating points, candidate nonnegative prices and a run-to-end option are replayed on nested source predictions. For each reported source budget target ($25\%,37.5\%,50\%,62.5\%,75\%$), selection minimizes source score shortfall among candidates whose mean declared cost is within that preference; if none is feasible, it uses the least-cost candidate. Ties prefer higher cost and then the smaller price parameter. The selected option is frozen for the target. Across the seven AutoLab source panels, running every trajectory to its recorded end uses $52.06\%$--$62.45\%$ of declared budget on average. Run-to-end therefore satisfies both the $62.5\%$ and $75\%$ source constraints with zero score loss and is selected for every target model; the other reported settings select nonnegative prices. Fixed-budget stopping uses the same preference as a declared cutoff, stopping at the last decision checkpoint no later than that cutoff. A desired source budget fraction need not equal the realized target cost fraction. All costs entering the rule are normalized by the declared budget of the run; normalization by eventual full-run duration is used only for reporting realized costs. EdgeBench replay uses published task--model mean curves; AutoLab uses individual replicates.

Relative score loss is $100L/\bar S_{\mathrm{full}}$, where $L$ is the reported mean loss in points and $\bar S_{\mathrm{full}}$ is the mean full-run score under the same task, replicate, and model averaging. The full-run means are $63.78$ points on AutoLab and $35.68$ on EdgeBench. Time saved is $100(1-\bar x)$, where $\bar x$ is the reported mean actual-cost fraction.

\begin{table}[htbp]
\centering
\caption{Frozen operating points. Each target supplies the conditioned policy's source budget constraint and the fixed policy's declared cutoff. Cost is a fraction of full-run time. Loss is in points, with relative score loss in parentheses. Results average replicates within model, then models.}
\label{tab:frozen_operating}
\small
\begin{tabular}{llrrrr}
\toprule
Target & Policy & \multicolumn{2}{c}{AutoLab} & \multicolumn{2}{c}{EdgeBench}\\
 & & Cost & Loss (pt; \%) & Cost & Loss (pt; \%)\\
\midrule
$25\%$ & Fixed budget & $46.3\%$ & $15.77$ ($24.73\%$) & $16.7\%$ & $10.12$ ($28.36\%$)\\
 & Conditioned & $46.5\%$ & $6.99$ ($10.96\%$) & $23.3\%$ & $7.66$ ($21.46\%$)\\
$37.5\%$ & Fixed budget & $46.3\%$ & $15.77$ ($24.73\%$) & $33.3\%$ & $5.95$ ($16.68\%$)\\
 & Conditioned & $67.2\%$ & $1.52$ ($2.39\%$) & $35.5\%$ & $4.96$ ($13.90\%$)\\
$50\%$ & Fixed budget & $70.1\%$ & $5.34$ ($8.38\%$) & $50.0\%$ & $3.74$ ($10.49\%$)\\
 & Conditioned & $89.0\%$ & $0.12$ ($0.19\%$) & $44.6\%$ & $3.47$ ($9.71\%$)\\
$62.5\%$ & Fixed budget & $70.1\%$ & $5.34$ ($8.38\%$) & $50.0\%$ & $3.74$ ($10.49\%$)\\
 & Conditioned & $100.0\%$ & $0.00$ ($0.00\%$) & $60.7\%$ & $1.87$ ($5.25\%$)\\
$75\%$ & Fixed budget & $87.2\%$ & $2.10$ ($3.30\%$) & $66.7\%$ & $2.14$ ($6.01\%$)\\
 & Conditioned & $100.0\%$ & $0.00$ ($0.00\%$) & $69.1\%$ & $1.19$ ($3.34\%$)\\
\bottomrule
\end{tabular}
\end{table}

\subsection{Sequential Reassessment with Imperfect Gain Estimates}
\label{sec:decision_proof}

Consider a finite declared grid $\tau_0<\cdots<\tau_N=1$, a bounded best-so-far potential score process $Y_k$, and an increasing filtration $\mathcal F_k$ containing its visible prefix and fixed source information. Stopping truncates this potential trajectory; natural termination is absorbing. The price $\lambda\ge0$ is fixed before the target run. Purchased intervals have additive declared costs $c_k=\tau_{k+1}-\tau_k$. A policy's continue indicator $a_k\in\{0,1\}$ is measurable with respect to $\mathcal F_k$, with forced stopping at natural termination and $N$.

For $k<N$, define the true conditional gain means and fixed-plan values as
\begin{align}
    g_{k,\mathrm n}&=\mathbb E[Y_{k+1}-Y_k\mid\mathcal F_k],
    &Q_{k,\mathrm n}&=g_{k,\mathrm n}-\lambda c_k,\\
    g_{k,\mathrm e}&=\mathbb E[Y_N-Y_k\mid\mathcal F_k],
    &Q_{k,\mathrm e}&=g_{k,\mathrm e}-\lambda(1-\tau_k),\\
    Q_k^{\max}&=\max(Q_{k,\mathrm n},Q_{k,\mathrm e}),
    &J_k&=\max(0,Q_k^{\max}).
\end{align}
At the endpoint set both $Q$ values to zero. At earlier absorbing states, the score gains are zero and the virtual endpoint value remains $-\lambda(1-\tau_k)$; $J_k=0$. These virtual costs retain the fixed-plan definition and are not charged to the stopped policy.

\begin{proposition}[Fixed-plan bound with sign loss]
\label{prop:sequential_value}
Let $V_k^\pi$ be the conditional expected score gain from checkpoint $k$ until policy stopping, minus $\lambda$ times the sum of purchased declared interval costs. At a visited checkpoint, define
\begin{equation}
    L_k=J_k-a_kQ_k^{\max}\ge0.
\end{equation}
The nonnegative sign loss is zero for a correct continue/stop decision. If $\mathcal V_k^\pi$ is the set of visited checkpoints starting at $k$, including the checkpoint where stopping is selected, then
\begin{equation}
    V_k^\pi\ge J_k-
       \mathbb E_\pi\!\left[\sum_{j\in\mathcal V_k^\pi}L_j\,\middle|\,\mathcal F_k\right].
    \label{eq:inexact_value_bound}
\end{equation}
In particular, if $a_j=\mathbf1\{Q_j^{\max}>0\}$ at every visited checkpoint $j$, then $V_k^\pi\ge J_k$.
\end{proposition}

\begin{proof}
Let $\mathcal D_k^\pi$ denote the conditional expected loss sum in equation~\ref{eq:inexact_value_bound}. At the endpoint and natural termination, $V_k^\pi=J_k=\mathcal D_k^\pi=0$. For an active checkpoint, additivity and the tower property yield
\begin{equation}
    Q_{k,\mathrm e}
      =Q_{k,\mathrm n}+\mathbb E[Q_{k+1,\mathrm e}\mid\mathcal F_k].
    \label{eq:conditional_tower}
\end{equation}
Proceed backward in $k$. If $a_k=0$, the policy stops, $V_k^\pi=0$, $L_k=J_k$, and there are no subsequent visits, giving equality. If $a_k=1$, the induction hypothesis implies
\begin{align}
    V_k^\pi
    &=Q_{k,\mathrm n}+\mathbb E[V_{k+1}^\pi\mid\mathcal F_k]\nonumber\\
    &\ge Q_{k,\mathrm n}+\mathbb E[J_{k+1}-\mathcal D_{k+1}^\pi\mid\mathcal F_k]\nonumber\\
    &\ge \max\!\left(Q_{k,\mathrm n},
               Q_{k,\mathrm n}+\mathbb E[Q_{k+1,\mathrm e}\mid\mathcal F_k]\right)
              -\mathbb E[\mathcal D_{k+1}^\pi\mid\mathcal F_k]\nonumber\\
    &=Q_k^{\max}-\mathbb E[\mathcal D_{k+1}^\pi\mid\mathcal F_k]
     =J_k-\mathcal D_k^\pi.
\end{align}
The second inequality uses both $J_{k+1}\ge0$ and $J_{k+1}\ge Q_{k+1,\mathrm e}$. The final equality uses $L_k=J_k-Q_k^{\max}$ when continuing and $\mathcal D_k^\pi=L_k+\mathbb E[\mathcal D_{k+1}^\pi\mid\mathcal F_k]$. This completes the induction.
\end{proof}

\paragraph{Effect of prediction error.}
Define $\epsilon_k=\max_h|\widehat g_{k,h}-g_{k,h}|$. Since the costs are identical in the true and estimated values, $|\max_h\widehat Q_{k,h}-Q_k^{\max}|\le\epsilon_k$. The estimated-value rule consequently satisfies
\begin{equation}
    L_k\le\epsilon_k\,\mathbf1\{a_k\ne\mathbf1\{Q_k^{\max}>0\}\}.
\end{equation}
Error contributes only when it changes the sign decision, and only along visited checkpoints. A positive numerical decision tolerance $\delta$ replaces $\epsilon_k$ in this inequality by $\epsilon_k+\delta$. Here $\epsilon_k$ includes finite-data fitting error, reference mismatch, and information lost by compressing the full prefix into seven features. The tower identity applies to the true conditional means; the separately updated fitted curves need not satisfy it exactly. The theorem compares against the three specified fixed plans and makes no assumption that the empirical regression attains the conditional means. If each purchased interval's realized cost is at most its declared interval cost, the same lower bound also holds for net value computed with realized costs.

\section{External Capability Representation}
\label{sec:representation_figure}

Figure~\ref{fig:capability_basis} displays the benchmark loadings and model coordinates used throughout the analysis. Table~\ref{tab:external_sources} identifies the corresponding external measurements and their sources.

\begin{figure}[tbp]
\centering
\includegraphics[width=\linewidth]{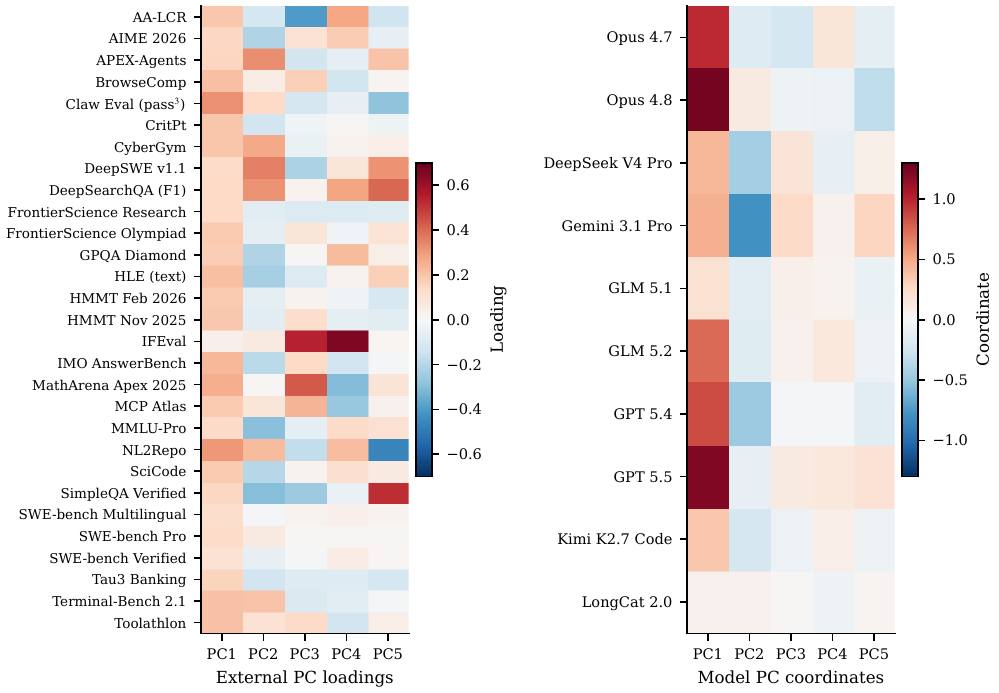}
\caption{External representation used by the capability--time model. Left: all $29$ benchmark loadings on five PCs estimated from $82$ reference model entries, excluding evaluated versions and configuration aliases. Right: the fixed coordinates of all ten evaluated model versions. Loadings and coordinates have separate color scales. The loadings summarize covariance patterns among external measurements; missing-entry estimation is specified in Appendix~\ref{sec:method_details}.}
\label{fig:capability_basis}
\end{figure}

\begin{table}[!htbp]
\centering
\caption{Sources of the 29 external measurements. Rows identify the versions or metrics used; links point to benchmark papers, dataset releases, or evaluation documentation.}
\label{tab:external_sources}
\begin{tabular}{@{}p{.43\linewidth}p{.53\linewidth}@{}}
\toprule
Measurement & Reference or source\\
\midrule
AA-LCR & \href{https://artificialanalysis.ai/methodology/intelligence-benchmarking}{Artificial Analysis methodology}\\
AIME 2026 & \href{https://huggingface.co/datasets/MathArena/aime_2026}{MathArena dataset}\\
APEX-Agents & \citet{vidgen2026apex}\\
BrowseComp & \href{https://github.com/openai/simple-evals/blob/652c89d0ca9df547706735883097e9537d40dc47/browsecomp_eval.py}{OpenAI evaluation release}\\
Claw Eval (pass$^3$) & \href{https://github.com/claw-eval/claw-eval/tree/5680b8b11ff2ee5dd2b07b89086a29a5c5c984d7}{Claw-Eval release}\\
CritPt & \href{https://huggingface.co/datasets/CritPt-Benchmark/CritPt}{CritPt dataset}\\
CyberGym & \href{https://www.cybergym.io/}{CyberGym project}\\
DeepSWE v1.1 & \href{https://github.com/datacurve-ai/deep-swe}{DataCurve benchmark release}\\
DeepSearchQA (F1) & \href{https://huggingface.co/datasets/google/deepsearchqa/tree/b2623f8653065c2672de6d941fc5434cd652376c}{Google dataset}\\
FrontierScience Research (FS Research) & \citet{wang2026frontierscience}\\
FrontierScience Olympiad & \citet{wang2026frontierscience}\\
GPQA Diamond & \citet{rein2024gpqa}\\
Humanity's Last Exam (text) & \href{https://labs.scale.com/leaderboard/humanitys_last_exam_text_only}{Text-only evaluation protocol}\\
HMMT February 2026 & \href{https://huggingface.co/datasets/MathArena/hmmt_feb_2026}{MathArena dataset}\\
HMMT November 2025 & \href{https://huggingface.co/datasets/MathArena/hmmt_nov_2025}{MathArena dataset}\\
IFEval & \href{https://arxiv.org/abs/2311.07911}{Instruction-Following Evaluation}\\
IMO-AnswerBench & \href{https://imobench.github.io/}{IMO-Bench project}\\
MathArena Apex 2025 & \href{https://matharena.ai/apex/}{MathArena Apex release}\\
MCPAtlas Public & \href{https://huggingface.co/datasets/ScaleAI/MCP-Atlas}{MCP-Atlas public dataset}\\
MMLU-Pro & \citet{wang2024mmlupro}\\
NL2Repo-Bench & \citet{ding2025nl2repo}\\
SciCode & \citet{tian2024scicode}\\
SimpleQA-Verified & \href{https://epoch.ai/benchmarks/simple-qa-verified}{Verified benchmark documentation}\\
SWE-bench Multilingual & \href{https://huggingface.co/datasets/SWE-bench/SWE-bench_Multilingual}{Multilingual dataset}\\
SWE-bench Pro & \citet{deng2025swepro}\\
SWE-bench Verified & \href{https://huggingface.co/datasets/SWE-bench/SWE-bench_Verified/tree/78f471bf655a3137b2e8a75af1501690ec009ec3}{Verified dataset}\\
$\tau^3$-Banking & \href{https://taubench.com/leaderboard/}{Sierra banking evaluation}\\
Terminal-Bench 2.1 & \href{https://www.tbench.ai/news/terminal-bench-2-1}{Terminal-Bench 2.1 release}\\
Toolathlon (original) & \href{https://toolathlon.xyz/}{Toolathlon project}\\
\bottomrule
\end{tabular}
\end{table}

\FloatBarrier
\section*{AI Use Statement}

AI tools provided iterative assistance with research and mathematical development, including hypotheses and proofs; experimental implementation and data analysis; and language polishing and consistency checks. Numerical results were checked against recorded experiment outputs. The authors take responsibility for the work.

\end{document}